\documentclass[acmtog]{acmart}

\usepackage{algorithm,algorithmic}
\usepackage{color}

\usepackage{newtxmath}
\usepackage{soul}
\usepackage{bm}
\usepackage[shortlabels]{enumitem}
\usepackage[utf8]{inputenc}

\def\eg{\textit{e.g.}}
\usepackage[export]{adjustbox}

\newtheorem{prop}{Proposition}
\usepackage{makecell}

\DeclareUnicodeCharacter{3000}{}
\AtBeginDocument{%
  \providecommand\BibTeX{{%
    \normalfont B\kern-0.5em{\scshape i\kern-0.25em b}\kern-0.8em\TeX}}}

\acmJournal{JACM}
\acmVolume{}
\acmNumber{}
\acmArticle{}
\acmMonth{11}

\begin{document}

\title{Adversarial Deep Learning for Online Resource Allocation}

\author{Bingqian Du}
\email{bqdu@hku.hk}
\affiliation{%
  \institution{The University of Hong Kong}
  \streetaddress{Pok Fu Lam}
  \country{Hong Kong}
}

\author{Zhiyi Huang}
\email{zhiyi@cs.hku.hk}
\affiliation{%
  \institution{The University of Hong Kong}
  \streetaddress{Pok Fu Lam}
  \country{Hong Kong}
}

\author{Chuan Wu}
\email{cwu@cs.hku.hk}
\affiliation{%
  \institution{The University of Hong Kong}
  \streetaddress{Pok Fu Lam}
  \country{Hong Kong}
}


\begin{abstract}
  Online algorithm is an important branch in algorithm design. Designing 
  online algorithms with a bounded competitive ratio (in terms of worst-case performance) can be hard and usually relies on problem-specific assumptions. Inspired by adversarial training from Generative Adversarial Net (GAN) and the fact that competitive ratio of an online algorithm is based on worst-case input, we adopt deep neural networks to learn an online algorithm for a resource allocation and pricing problem from scratch, with the goal that the performance gap between offline optimum and the learned online algorithm can be minimized for worst-case input. 

Specifically, we leverage two neural networks as algorithm and adversary respectively and let them play a zero sum game, with the adversary being responsible for generating worst-case input while the algorithm learns the best strategy based on the input provided by the adversary. To ensure better convergence of the algorithm network (to the desired online algorithm), we propose a novel per-round update method to handle sequential decision making to break complex dependency among different rounds so that update can be done for every possible action, instead of only sampled actions. To the best of our knowledge, our work is the first using deep neural networks to design an online algorithm from the perspective of worst-case performance guarantee. Empirical studies show that our updating methods ensure convergence to Nash equilibrium and the learned algorithm outperforms state-of-the-art online algorithms under various settings. 
\end{abstract}



\maketitle

\section{Introduction}
\label{sec:introduction}

Traditional algorithm design assumes that input will be revealed to the algorithm all at once. However, a great number of problems arising from reality do not fall into this category. Consider the resource allocation problem in a cloud computing platform. Users' requests for renting resources can arrive at any time; the platform needs to decide whether to rent resources out to the current user without the knowledge of future requests.   
Such a problem depicts an online setting, where decisions are made for partial inputs that have been revealed so far. Decisions are hard to make in such online settings because of future uncertainty. Still considering the cloud resource allocation problem, renting resource to the current user can be beneficial but 
may also take up resources which could be allocated to later users with much higher budgets. To deal with online problems, online algorithms have been studied for decades. One common metric to evaluate an online algorithm is the competitive ratio, which measures the gap between the performance of the offline optimum and the online algorithm in the worst case (i.e., under the adversary input) 
 \cite{sleator1985amortized}. 

Inspecting an online algorithm from the competitive ratio perspective, we can view the relationship between the online algorithm and the adversary input as two players at the Nash equilibrium of a zero sum game. For the ``algorithm'' player, it aims to minimize the gap between offline optimum and its own performance to achieve a good competitive ratio; for the ``adversary'' player, it targets maximizing this gap by generating hard cases for the algorithm to handle. When both of them arrive at the Nash equilibrium, any strategy change of the adversary will not cause worse performance of the algorithm than its performance at the Nash equilibrium (which corresponds exactly to the worst case in online algorithm analysis). The policy of the algorithm network at Nash equilibrium is thus the online algorithm with 
worst-case performance guarantee.

Generative Adversarial Net (GAN) \cite{goodfellow2014generative} has been a remarkable attempt to combine game theory and deep learning. In GAN architecture, there are two neural networks: a {\em generative} neural network $\mathit{G}$ and a {\em discriminative} neural network $\mathit{D}$. $\mathit{G}$ learns to map latent variables to data distribution $\digamma$ while $\mathit{D}$ tries to distinguish real data from $\digamma$ and data generated from $\mathit{G}$. $\mathit{G}$ and $\mathit{D}$ have opposite goals; the training process of $\mathit{G}$ and $\mathit{D}$ is equivalent to having $\mathit{G}$ and $\mathit{D}$ play a two-player zero-sum game. GAN has been proved effective in generating data distributions similar to real data in a number of successful applications to computer vision \cite{brock2018large}\cite{nie2017medical} and natural language processing \cite{zhang2016generating} \cite{yu2017seqgan}. 

Inspired by competitive ratio analysis of online algorithm and the GAN model, we investigate designing online algorithms using a deep learning method, instead of standard theoretical frameworks such as primal-dual \cite{buchbinder2009design} \cite{buchbinder2007online}. In this paper, we focus on online resource allocation and pricing for social welfare maximization, a classic category of online problems. Representing the online algorithm and the adversary input generation as two neural networks, we formulate their interaction as dynamics of two players in a zero sum game: the adversary generates worst-case input while the algorithm learns the best strategy based on the input provided by the adversary. 
The goal of the algorithm 
is to minimize the difference between offline optimal social welfare and the social welfare obtained by the online algorithm; 
the adversary 
maximizes this difference so that the worst case is ensured. Traditional online algorithm design typically considers the performance ratio between offline optimal solution and online solution; 
we use performance difference as the objective to ease problem formulation and update design, which also 
reflects the performance gap between offline optimal solution and online solution, and has been used in online algorithm literature before \cite{jansen2011approximation}. Since an online algorithm produces sequential decisions, the original neural network architectures and update methods of GAN can not be applied to our problem. The standard method for dealing with sequential decision making in deep learning literature is Reinforcement Learning (RL). We do not utilize RL as most of the existing works do since RL is known to be heavily relying on exploration-exploitation trade-off and can easily be trapped in sub-optimal solutions. In order to deal with sequences and achieve better convergence (to a good online algorithm), we carefully design a per-round update method for both algorithm and adversary neural networks from their respective optimization formulations, to break the strong correlation between different time steps in a sequence, so that no more exploration-exploitation heuristics are needed and update can be done for every possible action instead of just sampled actions during the training process.  
To the best of our knowledge, our work is the first to design worst case-based online algorithms using a deep learning approach. 

We carefully analyze the Nash equilibrium achievable by our approach. We also carry out careful empirical studies under different numbers of arriving users and resource units, and different user budget distributions. The results show the convergence to Nash equilibrium with our update methods and the superior performance of our learned algorithm compared with existing state-of-the-art online resource allocation and pricing algorithms.  In this paper, we explore the possibility of designing a worst case-based algorithm using a deep learning framework. We are aware that existing deep learning-based methods cannot provide formal theoretical guarantee and worst cases are rare in practice. However, the learned online algorithm and worst cases can provide insights to algorithm designer for better understanding the problem 
when worst cases are not obvious. We verify the effectiveness of our learning-based approach through empirical studies under both worst cases and random (common) cases.

\section{Related Work}
\label{sec:relatedwork}

\subsection{Deep learning for game playing}

Game playing is a classic area investigated by researchers for decades. With the prevalence of deep learning, it is natural to ask whether deep learning can be applied to solving complicated games. There have been several attempts in this regard. 
AlphaGo \cite{silver2017mastering} masters the game of Go without relying on any human knowledge; it combines Monte Carlo Tree Search and RL to improve its policy quality and policy evaluation accuracy with continuous self play. Green Security Game has been studied by Wang {\em et al.}~\cite{wang2019deep}, by training a Deep Q-Network (a kind of RL model) to learn an approximate best response. Yu {\em et al.}~\cite{yu2017seqgan} apply a GAN model for text generation; policy gradient and REINFORCE algorithm are used for the training of their generator, while the same training method of 
discriminator as in the original GAN paper \cite{goodfellow2014generative} is used in their discriminator training. Multi-agent deep reinforcement learning is utilised by Celli {\em et al.} \cite{celli2019coordination} to solve a sequential zero-sum game. Solving zero-sum game in linear discrete-time system is investigated by Luo {\em et al.} \cite{luo2020policy}. They develop a data-based policy iteration Q-learning algorithm to learn the optimal Q-function from data collected in real systems. A regret-based reinforcement learning algorithm is proposed by  Steinberger {\em et al.} \cite{steinberger2020dream} for imperfect-information multi-agent model-free setting in order to find the equilibrium. 
We can see that existing deep learning methods for 
game playing are mostly based on RL, but as pointed out in \cite{arulkumaran2017brief}, the main challenge faced by RL is that long-range time dependencies make the consequences of a single action unclear after many transitions of the environment, and its observations are limited by its sampled actions.  

We use a GAN-like framework; however, unlike GAN's one-shot output, our model deals with sequential decision making, which makes the training even more challenging. Based on the drawbacks of RL described before, we do not utilize the RL method as most of the existing studies do; instead, we design a novel per-round update method to tackle challenges in training of both algorithm and adversary neural networks, to ensure better convergence.

\subsection{Online resource allocation problems}

Online allocation and pricing for single-type, non-recycled resources is a fundamental online problem. It is equivalent to the classic online knapsack problem when pricing is not considered \cite{kellerer2004multidimensional}. The online knapsack problem was first studied by Marchetti-Spaccamela {\em et al.}~\cite{10.1007/BF01585758}, 
and they considered the average case analysis. Following that, 
Buchbinder {\em et al.}~\cite{buchbinder2009online} propose a general framework for design and analysis of online algorithms for packing problems (the knapsack problem falls into this category) based on two assumptions: (1) The budget per unit of resource of all users is lower bounded by $L$ and upper bounded by $U$; (2) The resource demand of a single user is infinitesimal compared to total resource number. Their algorithm gives an $\mathcal{O}(\log (U/L))$ competitive ratio with fractional solution; however, how to round the solution to integers is not clear. Based on the same assumptions, Zhou {\em et al.}~\cite{zhou2008budget} cast the online single ad slot auction problem as an online knapsack problem and propose a KP-Threshold algorithm which achieves a competitive ratio of $\ln (U/L) +1$. 

There have been recent studies on online resource allocation and pricing algorithm design. 
Online posted price and resource allocation under the same setting as ours was studied by Zhang {\em et al.}~\cite{zhang2017optimal}. Their deterministic online posted pricing and resource allocation algorithm is proved to be optimal under the same two assumptions as above. 
Unlike \cite{zhang2017optimal}, we do not make assumptions with respect to the relationship between resource demand and total resource supply and approach the online problem using a deep learning method. Besides, we provide a randomized algorithm, which can potentially capture more complicated features of the problem for better decision making than deterministic ones. 

 Some deep learning methods have been applied to solving online resource allocation and pricing problems. Tesauro {\em et al.}~\cite{tesauro2005online} use a decompositional RL method to learn a strategy for online server allocation. They empirically show that a RL-based method is comparable to a performance model framework based on queuing theory. Resource management in an online scenario is investigated by Mao {\em et al.} \cite{mao2016resource}. They translate the packing problem with resource constraints to a learning problem and propose to use RL for the learning process. Wang {\em et al.} \cite{wang2017automated} aim at achieving automated balance of performance and cost for cloud provisioning. By analysing the performance of RL under tabular, deep,
and dueling double deep Q-learning with the CloudSim simulator, they show the effectiveness of RL.  Du {\em et al.}~\cite{du2019learning} use a Long Short Term Memory (LSTM) neural network and a deep deterministic RL algorithm to learn resource allocation and pricing strategy for cloud computing platform, based on the assumption that user request sequences follow a fixed time series distribution. Existing learning-based methods typically make assumptions on the distribution of user requests, which contradicts worst-case input in competitive analysis of online algorithms. We approach online resource allocation and pricing using a deep learning method without making any assumption regarding user request sequence distribution, and worst-case performance is considered in this paper.

\subsection{Deep learning for classic theoretical problems}
Deep learning has been a new trend for solving classic optimization problems and has shown superior performance in computing time and solution quality. Dai {\em et al.}~\cite{khalil2017learning} propose a combination of reinforcement learning and graph embedding method to design algorithms for NP-hard combinatorial optimization problems. Empirical results show the effectiveness of their learning-based method. Travelling Salesman Problem is investigated via the neural approach by Bello {\em et al.} \cite{bello2016neural}. They utilize recurrent neural network and policy gradient method to tackle such a learning problem and achieve close-to-optimal
results on Euclidean graphs with up to 100 nodes.  Kool {\em et al.} \cite{kool2018attention} propose new models and training methods for routing problems, such as Travelling Salesman
Problem (TSP), Vehicle Routing Problem (VRP), etc. They design model-based attention layers to better capture the underlying problem structure and a greedy rollout-based REINFORCE algorithm to find optimal solutions. Evaluation shows their method outperforms a large number of
baselines and obtains results close to highly optimized and specialized algorithms. Graph-based combinatorial optimization problems are further studied by Li {\em et al.} \cite{li2018combinatorial} via a graph
convolutional network (GCN). The trained GCN is used to guide a tree search to generate a large number of good candidate solutions. Evaluation shows that their method performs on par with highly optimized state-of-the-art heuristic solvers. Unsupervised learning methods are considered by Karalias {\em et al.} \cite{karalias2020erdos}. By carefully designing the loss, they bridge the gap between the discrete feasible solution for the combinatorial optimization problem and the continuous output of the NN, which provides performance guarantee for the solution found by the NN. Empirical results show that their method outperforms heuristics and solvers in both solution quality and computation time.

Existing literature has shown the ability of neural approaches in solving classic optimization problems, mainly due to the powerful computation and the generalization ability of NN. By carefully designing the training method and the model, NN is a promising way to find high quality solution for hard optimization problems. Based on this observation, we believe the neural approach can be a novel tool for online algorithm design, with the goal of finding better solutions. To the best of our knowledge, our work is the first to tackle online optimization through a neural network based method.

\subsection{Online algorithm with advice}
Online algorithms usually assume no knowledge about future input, while with accumulated historical traces, it is likely that some information about the input can be incorporated to improve the performance of online algorithms. Medina {\em et al.}~\cite{medina2017revenue} utilise a predictor for reserve price optimization, which is the first work to relate the revenue gain with the quality of a machine learning predictor. Following this work, Lykouris {\em et al.}~\cite{lykouris2018competitive} study the caching problem with the augmentation of a machine learning oracle and prove an improved bound when the oracle has a low error rate, as compared with unconditional worst case. They propose {\em Robustness} and {\em Consistency} to evaluate the model performance with respect to the worst case and best prediction, respectively. Kumar {\em et al.}~\cite{kumar2018improving} improve the ski-rental algorithm and non-clairvoyant job scheduling using machine learning predictions. Their proposed algorithm achieves good trade-off between robustness and consistency, improves the performance with better predictions, and does not degrade much when predictions are worse. Gollapudi {\em et al.}~\cite{gollapudi2019online} also study the ski-rental problem. The main difference 
is that they consider multiple machine learning experts. 
They prove the improvement of their algorithm with the aid of prediction.  

Another line of works consider incorporating historical data traces directly to improve the performance of online algorithms. Prodan {\em et al.}~\cite{prodan2009prediction} propose a prediction-based method for resource provisioning, with predictions produced by a neural net trained using historical data traces. On top of the NN, they design generic analytical game load models for resource allocation. Vera {\em et al.}~\cite{vera2021online} propose a framework based on Bellman inequalities for designing online allocation and pricing algorithms. 
Following this work, Banerjee {\em et al.}~\cite{banerjee2020constant} utilise approximate dynamic programming based on Bellman inequalities,   
which resolve an offline relaxations to make the controller reduce its sensitivity to estimation error. 

Our work differs from these studies. In our work, we do not consider incorporating an online algorithm with machine learning prediction of input based on historical traces. Instead, we directly use the machine learning approach to learn the online algorithm and worst cases, still respecting the assumption that no information about the input is known to the online algorithm beforehand.

\subsection{Convergence rate analysis for differentiable games} There have been some recent efforts to study the convergence behaviour of GAN and other learning based games. Singh {\em et al.} \cite{singh2013nash} analyse the two-player two-action iterative general-sum game, where each agent updates its strategy through gradient descent. They prove that the strategies may not always converge, but the average payoffs always converge to the expected payoffs of NE. Nagarajan {\em et al.} \cite{nagarajan2017gradient} study the convergence dynamics of GAN by showing that the equilibrium point is locally asymptotically stable for GAN formulation. Based on this result, they further propose a regularization term to speed up convergence. Letcher {\em et al.} \cite{letcher2019differentiable} argue that gradient descent does not always converge to the local optimum of objective in GANs, when there are multiple interacting losses. By decomposing the game Jacobian, they propose Symplectic Gradient Adjustment to find stable fixed point. The non-asymptotic local convergence of two-player smooth game was studied by Liang {\em et al.} \cite{liang2019interaction}. They prove that the iteration needed for the Simultaneous Gradient Ascent to converge is dependent on the off-diagonal interaction term.

\section{Model And Training Methods}
\label{sec:model}
An overview of our method is given in Fig.~\ref{fig:arch}. We adopt two neural networks (NN), namely the {\em algorithm} NN and the {\em adversary} NN. The output (colored circle) of the algorithm NN is a probability distribution over the price set for each user, while the output from the adversary NN is probability distributions over the budget set for all users. The input (colored square) to the algorithm NN or the adversary NN is decided by the learned strategy of the other NN. To be more specific, in each training iteration, the price sequence as the input to the adversary NN and the budget sequence fed into the algorithm NN are sampled according to the output of the algorithm NN and the adversary NN, respectively. The price set A and the budget set B are assumed to be known and finite in our formulation, which can be estimated based on historical traces if not given. 

We now formulate our problem, model the algorithm and adversary neural networks, and derive the NN update methods for learning the online algorithm. Importance notations are summarized in Table ~\ref{notation_tab}. 
\begin{figure*}
    \centering
    \includegraphics[width = 0.8\textwidth]{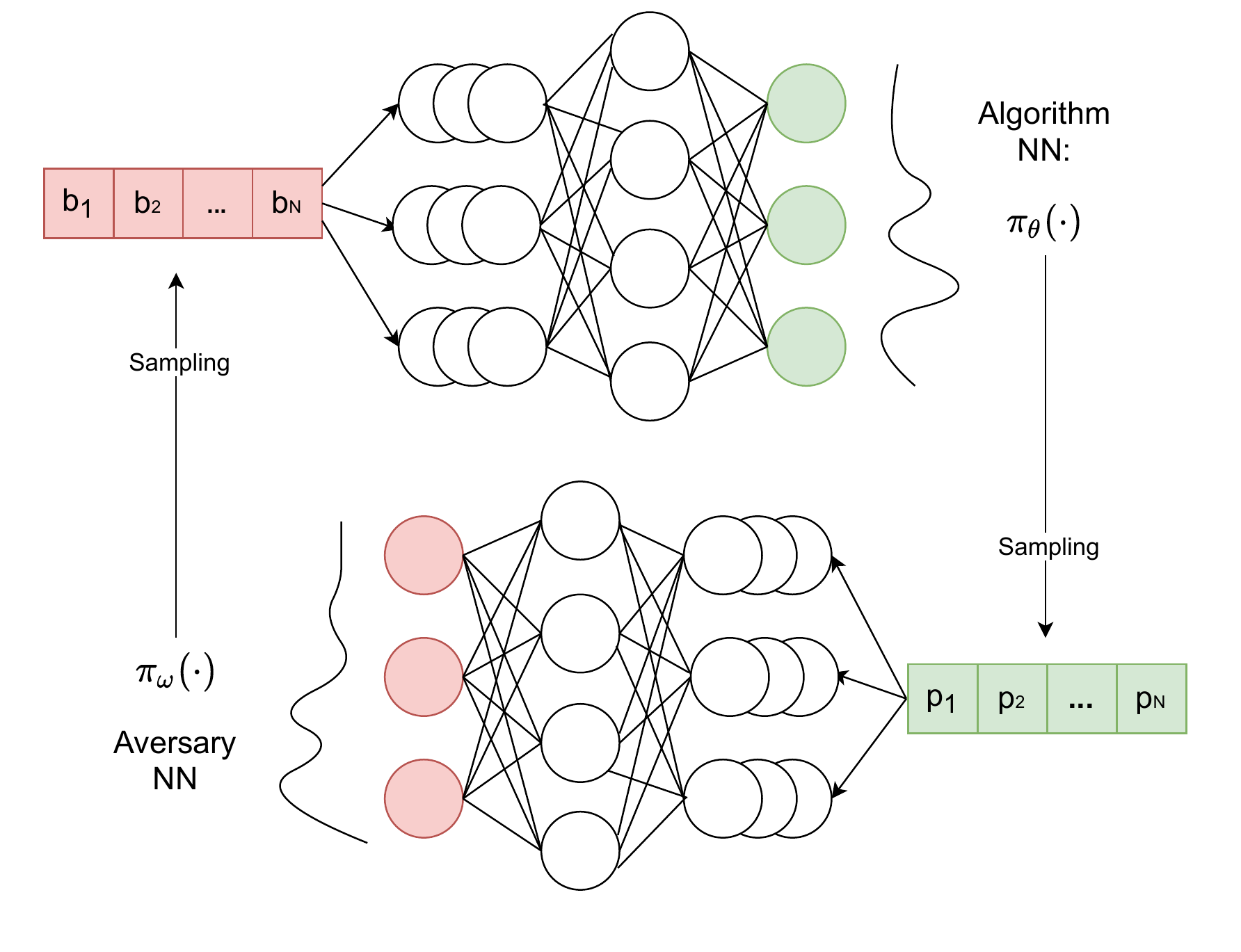}
    \caption{System Overview}
    \label{fig:arch}
\end{figure*}

\begin{table}[h]
    \centering
    \begin{tabular}{|c|c|c|c|}
    \hline
         $i$& user index & N &  $\#$ of users \\
         \hline
         $R$ & $\#$ of resource units & $b_i$ & budget of user $i$ \\
         \hline
         $B$ & budget set &$p_i$ & posted price for user $i$ \\
         \hline
         $A$ & price set & $\Omega$ & budget sequence set \\
         \hline
         $\pi_{\theta}(\cdot)$ & algorithm NN & $\pi_{\omega}(\cdot)$ & adversary NN \\
         \hline
         $\tau_{o}$ & joint mixed strategy  & $C$ & payoff matrix of the game \\
         \hline
         \multicolumn{1}{|c|}{$\tau_{p}$ } & \multicolumn{3}{c|}{ mixed strategy of the algorithm } \\
         \hline
         \multicolumn{1}{|c|}{$\tau_{b}$  } & \multicolumn{3}{c|}{ mixed strategy of the adversary  } \\
         \hline
         \multicolumn{1}{|c|}{$y_i$ } & \multicolumn{3}{c|}{ $\#$ of available resources when user $i$ arrives } \\
         \hline
          \multicolumn{1}{|c|}{$X_i$ } & \multicolumn{3}{c|}{ random variable denoting user $i$ is accepted or not }\\
          \hline
          \multicolumn{1}{|c|}{$x_i$ } & \multicolumn{3}{c|}{ realized acceptance decision for user $i$} \\
          \hline
          \multicolumn{1}{|c|}{$l_i$ } & \multicolumn{3}{c|}{ whether budget $i$ will be accepted in offline optimum} \\
          \hline
          \multicolumn{1}{|c|}{$Gap$ } & \multicolumn{3}{c|}{ the difference between offline optimum and online solution} \\
          \hline
          \multicolumn{1}{|c|}{$u_p(\cdot)$ } & \multicolumn{3}{c|}{ utility of algorithm: additive inverse of the expected gap} \\
          \hline
          \multicolumn{1}{|c|}{$u_b(\cdot)$ } & \multicolumn{3}{c|}{ utility of adversary: the expected gap} \\
          \hline
    \end{tabular}
    \caption{Notation Table}
    \label{notation_tab}
\end{table}
\subsection{Online Resource Allocation and Pricing Problem}

We consider social welfare maximization for single-type non-recycled resource allocation and pricing. There are $R$ units of resource supply in total. $N$ users arrive over time, each requesting one unit of the resource. The budget of user $i \in N$ is $b_i$, denoting how much the user is willing to pay for buying one unit of the resource. When user $i$ arrives, the online algorithm posts its price $p_i$ for one unit of the resource. We assume the algorithm is not aware of the current user's budget $b_i$, but knows posted prices and budgets of past users. If $p_i$ is no larger than $b_i$, user $i$ accepts the price and receives one unit of the resource. The portion of social welfare due to accepting user $i$ is $(b_i-p_i)+p_i=b_i$.

\subsection{Algorithm}

In the posted price scenario, 
 the accept or reject decision would normally be made by a user by comparing its budget to the current price. However, when formulating the optimization problem for the algorithm NN, 
 to allow the NN to adjust its pricing strategy, we instead consider the probability of a user being accepted or rejected by the NN as the decision variable. For the online algorithm, upon arrival of a user, the decision variable is a probability distribution over all possible prices. We use $X_i$ as the random variable denoting whether user $i$ is accepted or not.  $x_i$ is the realized acceptance decision of random variable $X_i$
: $x_i = 1$ if $b_i \ge p_i$ and there exists available resource, and $x_i = 0$, otherwise. $x_i = 1$ implies that one unit of resource is allocated to user $i$ while $x_i =0$ indicates that user $i$ does not consume any resource. Let $P(X_i=1|X_1 =x_1,..,X_j=x_j, 1 \le j <i)$ denote the probability of user $i$ being accepted conditioned on the realised acceptance of previous users, which is also the probability for choosing prices no larger than $b_i$. Similarly, $P(X_i=0|X_1 = x_1,..,X_j=x_j, 1 \le j < i)$ is the probability of not allocating resource to user $i$ conditioned on the realised decisions of previous users, {\em i.e.}, the probability for choosing prices larger than $b_i$. For simplicity of notation, we use $P(X_i=1|x_1..x_{j})$ and $P(X_i=0|x_1..x_{j})$ as a shorthand for $P(X_i=1|X_1 =x_1,..,X_j=x_j, 1 \le j <i)$  and $P(X_i=0|X_1 = x_1,..,X_j=x_j, 1 \le j < i)$, respectively. 

The goal of the online algorithm is to minimize the gap between offline optimal social welfare and the social welfare achieved by the algorithm, which is equivalent to maximizing the social welfare of the algorithm solely because the offline optimum is a constant for a given user budget sequence. 
The social welfare maximization problem to solve upon arrival of user $i$ can be formulated as follows (it is the offline optimization problem if user $i$ is the first user). Note that since rejecting one user will bring 0 social welfare increment and will not affect resource consumption, the terms related to $P(X_i=0|x_1..x_{j})$ are not shown in the following formulation and their gradient 
will be 0.  
\vspace{-10mm}

\begin{equation}\label{obj:maximization_conditioned_welfare}
\max_{\substack{\mathbf{P(X_j = 1|x_1..x_{i-1}),} \\ \forall j \in [i,N]}} f = \sum_{j=i}^{N} b_j P(X_j=1|x_1..x_{i-1})
\end{equation}

subject to:

\vspace{-10mm}

\begin{align*}
x_1+x_2+...+x_{i-1} + \sum_{j=i}^{N}P(X_j=1|x_1..x_{i-1}) \le R \quad\quad (\ref{obj:maximization_conditioned_welfare}a)\\
P(X_j=1|x_1..x_{i-1}) \in [0, 1], \forall j \in [i,N]\quad\quad (\ref{obj:maximization_conditioned_welfare}b)
\end{align*}

The objective function (\ref{obj:maximization_conditioned_welfare}) is the expected social welfare achieved by the algorithm conditioned on the decisions before user $i$.  (\ref{obj:maximization_conditioned_welfare}a) is the resource constraint, which bounds the expected resource consumption by the number of available resources, due to the randomization nature of our algorithm NN. The conditional probability based on previous realized decisions would ensure the optimality of the complete sequence, that users with the largest R budgets would be accepted. (\ref{obj:maximization_conditioned_welfare}b) presents the decision variables in the problem. 

The above problem is a linear program, where strong duality holds. We can relax constraint (\ref{obj:maximization_conditioned_welfare}a) by introducing Lagrangian multiplier $\lambda$, and obtain the following Lagrangian function:

\begin{equation}
\begin{split}
 & \mathcal{L}(P(X_j=1|x_1..x_{i-1}),\lambda) =   \hspace*{1mm}\sum_{j=i}^{N} b_j P(X_j=1|x_1..x_{i-1})
 \\& + \lambda (R-x_1-..-x_{i-1}-\sum_{j=i}^N P(X_j=1|x_1..x_{i-1})) \\
 &= \sum_{j=i}^{N}(b_j-\lambda)P(X_j=1|x_1..x_{i-1}) + \lambda (R-x_1-..-x_{i-1})    
\end{split}
\end{equation}

The dual function is then
\vspace{-5mm}

$$\mathcal{G}(\lambda) = \max_{\substack{\mathbf{P(X_j = 1|x_1..x_{i-1}),} \\ \forall j \in [i,N]}}\\ \mathcal{L}(P(X_j=1|x_1..x_{i-1}),\lambda)$$



Let $OPT$ be the optimal objective value of the primal problem, and $\lambda^*$ represent $(\lambda^*|x_1..x_{i-1})$.  Since strong duality holds, we have:

\begin{align*}
  OPT& = \mathcal{G}(\lambda^*) = \max_{\substack{\mathbf{P(X_j = 1|x_1..x_{i-1}),} \\ \forall j \in [i,N]}} \mathcal{L}(P(X_j=1|x_1..x_{i-1}),\lambda^*)\\
   & = \max_{\substack{\mathbf{P(X_j = 1|x_1..x_{i-1}),} \\ \forall j \in [i,N]}} \sum_{j=i}^{N} (b_i-\lambda^*) P(X_j=1|x_1..x_{i-1}) 
  \\& + \lambda^* (R-x_1-..-x_{i-1})
\end{align*}

Given $\lambda^*$, $\lambda^*(R-x_1-..-x_{i-1})$ is a constant, and solving the primal problem is equivalent to solving:

\begin{equation}\label{equi_obj}
\max_{\substack{\mathbf{P(X_j = 1|x_1..x_{i-1}),} \\ \forall j \in [i,N]}} \hspace*{2mm}\sum_{j=i}^{N} (b_i-\lambda^*) P(X_j=1|x_1..x_{i-1})
\end{equation}

subject to:

\begin{align*}
    P(X_j=1|x_1..x_{i-1}) \in [0, 1],\forall j \in [i,N].\quad\quad(\ref{equi_obj}a)
\end{align*}

The gradient of function (\ref{equi_obj}) on $P(X_j=1|x_1,..x_{i-1})$ is $b_i -(\lambda^*|x_1..x_{i-1})$. 
According to Simon {\em et al.}~\cite{simon1994mathematics}, $(\lambda^*|x_1..x_{i-1}) = \frac{\partial f(P^*(X_j=1|x_1..x_{i-1}))}{\partial R}$ measures the rate of the change of the optimal value of problem~(\ref{obj:maximization_conditioned_welfare}) with respect to resource capacity (in constraint~(\ref{obj:maximization_conditioned_welfare}a)), {\em i.e.}, how much the objective function value will increase if another unit of resource is available. We will use this gradient in the stochastic gradient descent (SGD) method in solving the optimization problem, via training an algorithm neural network. For algorithm NN training, we know complete budget sequences, as provided by the adversary neural network. Given an acceptance decision sequence $x_i,...,x_N$ and the observation that optimal Lagrangian multiplier represents the infinitesimal change in the optimal social welfare with one more unit of resource, $(\lambda^*|x_1..x_{i-1})$ can take any value between the $y^{\mbox{th}}$ budget and the $(y+1)^{\mbox{th}}$ budget in the ordered budget sequence in decreasing budget order, with $y =y_i = R - \sum_{l=1}^{i-1}x_l$ being the current amount of available resources upon arrival of user $i$. We set $(\lambda^*|x_1..x_{i-1})$ to be the average of the $y^{\mbox{th}}$ budget and the $(y+1)^{\mbox{th}}$ budget.

 We provide more discussions to better justify our update method for the algorithm NN. 
\begin{itemize}
    \item Conditional probability: the problem formulation (\ref{obj:maximization_conditioned_welfare}) for the algorithm considers the probability of accepting a user, conditioned on the acceptance realization of previous users. The conditional probability is to avoid the situation that $R^{\mbox{th}}$ budget and the $(R+1)^{\mbox{th}}$ budget are the same so that the update signal for both of them would be 0, which would result in sub-optimal solution. Specifically, if the decision variable is $P(X_i)$, denoting the probability of accepting user $i$, the problem can be formulated as 
    \begin{align}
          \max_{P(X_i), \forall i \in [1,N]}  \sum_{i=1}^N & b_j P(X_i) \\
        \text{subject to:} 
        \sum_{i=1}^N P(X_i) \le R, P_i \in [0,1], & \forall i \in [1,N]
    \end{align}
    Recall that given available resource number $R$, $\lambda^*$ can take the average of the $R^{\mbox{th}}$ budget and the $(R+1)^{\mbox{th}}$ budget. If the $R^{\mbox{th}}$ budget and the $(R+1)^{\mbox{th}}$ budget are the same, the gradient ($b_i - \lambda^*$) for both of them would be 0 while the optimal solution should accept one of them. To avoid this case, we realize the acceptance decision $x_i$ for the current user $i$ and consider conditional probability so that optimal solution of the algorithm formulation can be ensured for the current budget sequence. 
    \item Greedily optimizing w.r.t sampled budget sequence: The problem the algorithm tries to solve is $\min_{alg} \max_{adv} Gap$. In the framework proposed, the adversary is adjusting its strategy to provide a worst sequence with the largest gap. To solve this problem, the algorithm should adjust its strategy to minimize the gap/maximize social welfare for the worst case sampled by the adversary. We note that the NE strategy may not be the optimal solution for each single sequence, so the optimizing step taken for each sequence could make the strategy of the NN deviate from the NE strategy; but on average the value achieved by algorithm NN 
    would be equivalent to the value of NE when the two NNs are playing the zero sum game (which is also the reason why we save the last 1k training results to obtain the average performance in the evaluation section).
\end{itemize}

\subsection{Adversary}\label{sec: adversary}

We consider an oblivious adversary, which generates the complete worst-case input sequence before the sequence is handled by online algorithm \cite{ben1994power}.
Oblivious adversary is the most studied type of adversary, 
which is popular in practical settings since most of the time, input sequence is independent of the algorithm. \eg, in the ski rental problem, the weather is independent of the choice of buying or renting skis. 

Given the pricing strategy of the online algorithm, the goal of our adversary is to generate the user arrival sequence (aka budget sequence) 
to maximize the 
gap between the offline optimal social welfare and the social welfare achieved by the online algorithm. The complete budget sequence has a length of $N$; for the $i^{\mbox{th}}$ slot in the sequence, the adversary produces a probability distribution $P_i(\cdot)$ over budget set $B = \{b_1, b_2,...b_m\}$, which contains all possible budget choices, where $m$ is the number of possible budget values. For each user, the budget is independently chosen according to the probability distribution $P_i(\cdot)$ over all budgets.

Given the price sequence produced by the algorithm,  $p_1, p_2,...,p_N$, the 
 optimization problem of the adversary can be formulated as follows:

\begin{equation}\label{obj:maximization_conditioned_gap}
\max_{\substack{\mathbf{P_i(b_{i}^{(j)})}\\ i \in [1,N], j \in [1, |\Omega|]}} \sum_{ j \in [1, |\Omega|]} \prod_{i=1}^N P_i(b_i^{(j)}) Gap^{(j)}
\end{equation}

subject to:

\begin{align*}
Gap^{(j)} = \sum_{i=1}^N (b_{i}^{(j)} l_i^{(j)} - b_{i}^{(j)} \vmathbb{1}(b_{i}^{(j)} \ge p_{i}) \vmathbb{1}(y_{i}^{(j)} > 0)), \forall j \in [1, |\Omega|] ~(\ref{obj:maximization_conditioned_gap}a)\\
l_{i}^{(j)} = 
\begin{cases}
    1,& if ~ b_i^{(j)} \in { benchmark}\\
    0,              & \text{otherwise}
\end{cases}, \forall i \in [1,N], j \in [1, |\Omega|] ~(\ref{obj:maximization_conditioned_gap}b)\\
 y_{i}^{(j)} = R - \sum_{l=1}^{i-1} \vmathbb{1}(b_{l}^{(j)} \ge p_{l}), ~\forall i \in [1,N], j \in [1, |\Omega|] ~(\ref{obj:maximization_conditioned_gap}c)\\
\end{align*}

\noindent Here $(b_{1}^{(j)}..b_{N}^{(j)})$ is a combination of $N$ elements from set $B$. We use $\Omega$ to denote the set of all such combinations, which is also the set of all possible complete budget sequences of the adversary. $j \in [1, |\Omega|]$ is used to denote one specific combination/ budget sequence, where $|\Omega|$ is the number of all possible combinations. For the complete budget sequence with index $j$, $b_i^{(j)}$ is the budget choice for user $i$. 
$P_i(b_i^{(j)})$ denotes the probability of choosing budget value $b_{i}^{(j)}$ for user $i$. 
 The probability of choosing the complete budget sequence $j$ equals the product of the probabilities of choosing each budget $b_i^{(j)}$ in the sequence. 
 The objective function in (\ref{obj:maximization_conditioned_gap}) gives the expected gap over all possible budget sequences of the adversary. Note that since different budget sequences are generated using the same adversary neural network, we have $P_i(b_i^{(j)}) = P_i(b_i^{(j^{\prime})})$ as long as $b_i^{{j}} = b_i^{(j^{\prime})}$.
 
$Gap^{(j)}$ is defined in (\ref{obj:maximization_conditioned_gap}a), with $l_{i}^{(j)}$ and $y_i^{(j)}$ defined in (\ref{obj:maximization_conditioned_gap}b) and (\ref{obj:maximization_conditioned_gap}c), respectively.
$l_{i}^{(j)} =1$ indicates that budget $b_i^{(j)}$ is counted into the offline optimal social welfare (referred to as the {\em benchmark}), which implies that $b_i^{(j)}$ is among the top-$R$ largest budgets in sequence $j$, {\em i.e.}, user $i$ with this budget is allocated with one unit of resource in the offline optimal solution, and $l_{i}^{(j)} =0$, otherwise. 
$y_i^{(j)}$ is the number of available resource units upon arrival of user $i$ when budget sequence $j$ is discussed. A user will be allocated one unit of resource by the algorithm if and only if : (1) his budget is no smaller than the posted price by the algorithm, {\em i.e.}, $b_{i}^{(j)} \ge p_{i}$; (2) there are available resources to allocate, {\em i.e.}, $y_{i}^{(j)} > 0$. The gap between the offline optimal social welfare and the social welfare of the online algorithm can be computed by summing up the difference between the chosen budget for each user (if it is in the benchmark) and the budget if it is accepted by the algorithm in the respective slot, as in  (\ref{obj:maximization_conditioned_gap}a).

 Unlike the optimization problem on the algorithm side, it is hard to relax 
constraints to derive the gradient of the objective function on $P_i(b_i^{(j)})$ in the above adversary's optimization problem. It is because the numbers of accepted users in the benchmark and by the algorithm are both constrained by the resource number, and directly 
enumerating all possible combinations in $\Omega$ to obtain the gradient on $P_i(b_i^{(j)})$ can soon be intractable when the user sequence becomes long and the budget set becomes large. Nonetheless, the algorithm-side gradients can be used as a reference for the derivation of gradients on the adversary side. Since $(\lambda^*|x_1...x_{i-1})$ measures the change of the optimal objective value of the algorithm when one more unit of resource is available, it is a tight lower bound of all accepted user budgets from user $i$ to $N$ in the optimal solution of the algorithm, derived upon user $i$'s arrival. So the gradient on the probability $P(X_j=1|x_i\ldots x_{i-1})$ of accepting user $i$, $b_i-\lambda^*$, measures 
the difference 
 between the objective value computed when the decision of accepting user $i$ is made by the algorithm and the objective value calculated when the optimal choice for user $i$ is made, assuming optimal choices were adopted in all later slots ($j>i$) (so we only need to compare $b_i$ with $\lambda^*$). We provide a simple example for better illustration. Consider the following user budget sequences: $\{1,3,3,3\}$. Suppose the available resource number is 2, and then $\lambda^* = 3$. The gradient for accepting the first user is 1-3, which equals the difference between (1+3)-(3+3), that is (the objective value when accepting the first user and taking optimal solution for latter slots)-(the objective value of the optimal solution).

Since the gradient of the algorithm's objective function on the probability of choosing a price at slot $i$ measures the effect of the price choice at slot $i$ to the objective value when all later slots are fixed to optimal choices, following this principle, on the adversary side, the gradient on the probability of choosing budget $b_l$ for user $i$ ($P_i(b_l)$) should reflect the objective value of the adversary when budgets of all later slots after $i$ are fixed to optimal budgets and the budget of user $i$ is set to $b_l$. We hence adopt such a heuristic approach for more efficiently computing gradients of the adversary's optimization problem. To derive such gradients, the next question is how to compute the optimal budgets for unrealized slots, since solving the adversary's optimization problem for those slots is not as obvious as solving the algorithm's optimization problem. 

Given partially realised budget sequence and complete price sequence, we can calculate the optimal budget choices for unrealised slots in polynomial time ({\em i.e.}, solving the optimization in (\ref{obj:maximization_conditioned_gap}) with budgets in the first $i-1$ slots realized). 
We only need to consider two cases for gap maximization:   
\begin{enumerate}[(a)]
    \item For each user $j\in[i,N]$, set the budget to be the largest budget in $B$ that is smaller than price $p_j$ in algorithm's price sequence, or post the smallest budget in $B$ if no budget in $B$ is smaller than price $p_j$.
    \item Suppose the number of unrealised slots is $U=N-i+1$ and $k$ is the current number of available resource units. Consider each slot index $j \in[k+i, N]$: For slots between $[i, j-1]$, find $k$ slots with the smallest prices according to the algorithm's price sequence, set the smallest budget from $B$ that is no less than the respective price for each of these slots (such that these users will be accepted but with smallest social welfare increment)\footnote{If all budget values are smaller than respective prices at these $k$ slots, we will not consider the current $j$ any more and will continue to the next $j$.}; supposing the index of the last slot among these $k$ accepted slots is $j^\prime$, for slots after $j^\prime$, set the largest budget from $B$, and for slots before $j^\prime$ other than the $k$ slots picked above, set the respective budget the same way as described in case (a) (such that corresponding users are maximally rejected to increase benchmark value). In this way, we obtain $(U-k)$ budget sequences, and the one with the largest gap will be the budget sequence output in this case.
\end{enumerate}

(a) above represents the case that almost no resource is allocated starting from user $i$ onward; (b) corresponds to the case that all remaining resources are allocated in the following. We do not need to consider 
other scenarios where part of the remaining resources are allocated in the unrealised slots, which will always lead to a smaller gap than that in case (a). Consider the following situation when one budget value is different from that in the respective slot in case (a) and the corresponding user is accepted by the algorithm, resulting in one more unit of resource usage (or the same amount of resource usage if case (a) has already used up all resources) as compared to (a): (i) if the slots whose budgets are counted into the benchmark are not affected, since social welfare achieved by the algorithm is increased by the accepted budget value, 
then the gap between the benchmark and the algorithm will be smaller; (ii) if the new budget value allows the corresponding slot to be counted into the benchmark, replacing another slot with a smaller budget, no matter whether the user of this replaced 
budget 
is accepted or rejected by the algorithm, the gap will always not be larger after replacement. Suppose a user's budget $b^\prime$ is replaced by another user's budget $b^*$ in the benchmark; $b^* > b^\prime$ due to the replacement in the benchmark. If the user with budget $b^\prime$ is accepted by the algorithm after replacement (it definitely was accepted by the algorithm before replacement due to the same price and budget for this slot before replacement), the change of the gap is $(b^*-b^\prime-b^*)-(b^\prime-b^\prime) = -b^\prime$; 
 if $b^\prime$ is rejected by the algorithm after replacement, there are two possible scenarios: (1) $b^\prime$ is also rejected before replacement, and then the change of the gap in this scenario is $(b^* - b^*)-b^\prime = -b^\prime$; (2) $b^\prime$ is accepted before replacement, but is rejected due to resource exhaustion by $b^*$ after replacement, and then the corresponding gap is $(b^* - b^*)-(b^\prime-b^\prime) = 0$. 

The complete algorithm for computing optimal budget sequence of adversary given the price sequence of algorithm is Algorithm~\ref{alg_adv_opt}. The input contains complete price sequence from the algorithm side, partially realised budget sequence, total number of resource units and the budget set. Lines 4-5 compute the number of available resources and social welfare of the algorithm for the partially realised budget sequence part. Case (a) 
is considered in lines 7-12, 
and case (b) is implemented from line 13 to line 26: every possible resource running out situation is considered (line 13), where the available resource number of slots with smallest prices will be accepted with the smallest possible budgets (line 14-18) while the highest possible budgets causing algorithm rejection will be set for other slots (line 14); the highest budgets will be set for slots after using up resources (lines 21-22). Amongst all situations, one with the largest gap will be the final solution in case (b) (lines 25-26).

\begin{algorithm}[h]
\caption{OPT\_BUDGET}
\label{alg_adv_opt}
\begin{flushleft}{\bf Input:} 
Complete price sequence $P$: $p_1,..,p_N$ ;
Partially realised budget sequence: $b_1^\prime,..b_{l-1}^\prime$; Total resource number $R$; Budget set $B=\{b_1,b_2,..b_m\}$\\
 {\bf Output:} 
Optimal budgets for unrealised slots $b_l^*,..b_N^*$ so that complete budget sequence $b_1^\prime,..b_{l-1}^\prime, b_l^*,..b_N^*$ with largest gap can be produced
\end{flushleft}
\begin{algorithmic}[1]
\STATE opt\_gap = 0, res = R, alg\_perf =0,  rej\_b\_seq = \{$b_1^\prime,...,b_{l-1}^\prime, 0,...,0$\} (with $N$ slots)
\STATE ~~~$Gap_{rej}$ = 0\COMMENT{variables for case (a)}
\STATE  acc\_b\_seq = \{\}, $Gap_{acc}$ = 0 \COMMENT{variables for case (b)}
\FOR{$i=1$ to $l-1$}
    \STATE If $b_i^\prime \geq p_i$ \textbf{and} res~$ > 0$:
      alg\_perf += $b_i^\prime$, res -= 1
\ENDFOR

\FOR{$i=l$ to $N$}  \STATE temp\_b = min($B$)
　　\FOR{$j=1$ to $m$}
　　    \STATE If $b_j<p_i$ \textbf{and} temp\_b  $<b_j$: temp\_b = $b_j$ \COMMENT{make users rejected} 
　　\ENDFOR
　　\STATE rej\_b\_seq[i] = temp\_b
　　\IF{$i-l+1 \geq res$} 
　　    \STATE temp\_b\_seq  = rej\_b\_seq, find $res$\# of smallest price from $p_l,...,p_i$, mark their index as $a1$ to $ar$ \COMMENT{use up resources}
　　    \FOR{$k=a1$ to $ar$}
　　        \STATE temp\_b\_seq[k] = $\infty$
　　        \FOR{$j =1$ to $m$}
　　            \STATE If $b_j \geq p_k$ \textbf{and} $b_j<$ temp\_b\_seq[k]: temp\_b\_seq[k] = $b_j$
　　        \ENDFOR
　　    \ENDFOR
　　    \FOR{$o=ar$ to $N$} 
　　         \STATE temp\_b\_seq[o] = max($B$)
　　    \ENDFOR
　　\ENDIF
　　\STATE $Gap$= Benchmark(temp\_b\_seq) - Alg(temp\_b\_seq, $P$)
　　\STATE If $Gap > Gap_{acc}$: acc\_b\_seq = temp\_b\_seq, $Gap_{acc}= Gap$
\ENDFOR

\STATE $Gap_{rej}$ = Benchmark(rej\_b\_seq) - Alg(rej\_b\_seq, $P$)
\STATE If $Gap_{acc} > Gap_{rej}$: \RETURN $Gap_{acc}$, acc\_b\_seq
\STATE Else: \RETURN $Gap_{rej}$, rej\_b\_seq
\end{algorithmic}
\end{algorithm}

\subsection{Training neural networks}


Let $\pi_\theta(h_{i-1}, x_i)$ represent the neural network model of the algorithm, where $\theta$ is the set of parameters in this algorithm NN, $h_{i-1}$ is the history information  before user $i$ and $x_i$ is received information of new user $i$. $h_{i-1}$ and $x_i$ are input to the algorithm neural network $\pi_\theta(\cdot)$. In our implementation, $x_i$ consists of the current index $i$, current available resource number $y_i$, last user's budget $b_{i-1}$, and last realised price $p_{i-1}$; $h_{i-1}$ is the encoded history before user $i$ using $x_1$ to $x_{i-1}$, computed by $h_{i-1} = vec[\delta([x_1,...x_{i-1}]*W+\bm{b})]$, where $W$ is a matrix to give different weights to different slots and information, $\bm{b}$ is a bias matrix, $\delta(\cdot)$ is an activation function, $vec[\cdot]$ is to flat the result 
to a vector so that $h_{i-1}$ and $x_i$ can be concatenated together as input to the algorithm NN. The output of the algorithm neural network $P_i^{\theta}$ is a probability distribution over all possible price choices for user $i$. 

Let $\pi_\omega(v)$ represent the neural network model of the adversary, where $\omega$ is the set of parameters in this adversary NN and $v$ is the input to the adversary NN.  
 $v$ is a vector of latent variables sampled from some prior distribution $p(v)$; we use the Gaussian distribution in our implementation, the same as in GAN \cite{mirza2014conditional}. The output $P^{\omega}$ is $N$ probability distributions over budget set $B$, to produce the budget sequence. 

The complete algorithm for training the two NNs is given in Algorithm \ref{alg_train}. 
Suppose 
the size of the price set 
is $n$ and the size of the budget set is $m$. At each iteration $t$, we first sample a batch of $s$ latent variables, budget sequences and price sequences (lines 3-5). When updating the algorithm NN (line 17), we calculate the gradient of $f$ in (\ref{obj:maximization_conditioned_welfare}) on parameter $\theta$, in order to optimize $f$. According to the chain rule: $\frac{\partial f}{\partial \theta} = \frac{\partial f}{\partial P^\theta(p)} \times \frac{\partial P^\theta(p)}{\partial \theta}$, where $P^\theta(p)$ is the probability output of the algorithm NN. The gradient of $f$ on the probability of accepting user $i$ is $b_i^{(j)} - (\lambda^*|p_1^{(j)}..p_{i-1}^{(j)})$ if $p_l \le b_i^{(j)}$ and 0 if $p_l > b_i^{(j)}$, as discussed in Sec.~3.2. 

Similarly, 
fix a price sequence $j$, the gradient of the adversary objective on the probability of choosing budget $b_l$ at user $i$ equals the gap value (line 10) when all slots after $i$ are fixed to optimal budget choices (line 9). 
Adversary NN update is performed based on the cumulative gradient computed from the entire batch (line 12). 

Note that in  
each training round, we can calculate the gradient of the objective on the probability of any possible output price/budget choice, so that all prices'/budgets' probabilities can be updated, instead of only sampled outputs as in standard reinforcement learning. 

\begin{algorithm}[h]
\caption{Stochastic gradient descent training of algorithm  and adversary NNs}
\label{alg_train}
\hspace*{0.02in} 
\begin{algorithmic}[1]
\FOR{t=1, 2, \ldots}
    \FOR{$\xi$ steps}
        \STATE Sample $s$ latent variables $\{v^{(1)},...,v^{(s)}\}$ from  prior $p(v)$
        \STATE Sample $s$ budget sequences \{$b^{(1)}$,...,$b^{(s)}$\} by feeding sampled latent variables to adversary NN
        \STATE Sample $s$ price sequences \{$p^{(1)}$,...,$p^{(s)}$\} by feeding sampled budget sequences to algorithm NN
        \STATE Update Adversary NN: $gradient=0$
    
         \FOR{$j=1$ to $s$, $i=1$ to $N$, $l=1$ to $m$}
        
                \STATE $b_i^{(j)}=b_l$
                \STATE $gap^{(j)}, seq^{(j)}$ = OPT\_BUDGET ($p^{(j)}, \{b_1^{(j)},..,b_{i}^{(j)}\}, R, B$)
                \STATE $gradient \mathrel{+}=  $ $gap^{(j)}\nabla_{\omega^{t-1}}P_i^{\omega^{t-1}}(b_l)$
       \ENDFOR
    
        \STATE $\omega^{t-1} = \omega^{t-1} + gradient$
    \ENDFOR
    \STATE $\omega^t = \omega^{t-1}$
    \STATE Sample $s$ latent variables $\{v^{(1)},...,v^{(s)}\}$ from  prior $p(v)$
    \STATE Sample $s$ budget sequences \{$b^{(1)}$,...,$b^{(s)}$\} by feeding sampled latent variables to adversary NN
    \STATE Update algorithm NN: 
    $\theta^t = \theta^{t-1}+ \sum_{j=1}^s \sum_{i=1}^{N} \sum_{l=1}^{n} (b_i^{(j)} - (\lambda^*|p_1^{(j)}..p_{i-1}^{(j)})) \vmathbb{1}(p_l \le b_i^{(j)}) \nabla_{\theta^{t-1}}P_i^{\theta^{t-1}}(p_l)$

\ENDFOR    
\end{algorithmic}
\end{algorithm}



\section{Theoretical Analysis}
\label{theory}

\subsection{Existence of Nash equilibrium}
We next discuss the existence of Nash equilibrium (NE) of the game played by the algorithm and the adversary, and the convergence to the NE by our algorithm and adversary NN training.

Suppose the price set of algorithm is $A = \{p_1,p_2,...,p_n\}$ and the budget set of adversary is $B = \{b_1,b_2,...b_m\}$. There are in total $N$ users. We use $\mathbb{A} = \{\alpha_1, \alpha_2,...\alpha_{n^N}\}$ and $\mathbb{B} = \{\beta_1, \beta_2,..., \beta_{m^N}\}$ to denote the pure strategy set of the algorithm and the adversary, respectively, where $\alpha_l, \forall l \in [1, |\mathbb{A}|]$, contains $N$ prices chosen from set $A$ and $\beta_l, \forall l \in [1, |\mathbb{B}|]$, contains $N$ budgets from set $B$. 
A mixed strategy of a player is a random distribution 
over its pure strategies. The set of such mixed strategies is denoted by $\tau_p$ and  $\tau_b$ for the algorithm and the adversary, respectively. The joint mixed strategy set is $\tau_o = \tau_p \otimes \tau_b$. 

For ease of presentation, we use $u_{p}(\cdot)$ to represent the additive inverse of the expected gap 
as the algorithm's utility and $u_{b}(\cdot)$ to represent the expected gap 
as the adversary's utility, so that the goal for both algorithm  and  adversary  is  to maximize their own utility (recall that the algorithm's objective is to maximize the expected social welfare it achieves which is equivalent to maximizing the additive inverse of the expected gap).

A joint mixed strategy $\tau_o^* \in \tau_o$ is the Nash equilibrium , if 
the following holds, where $\tau^*_{po}$ ($\tau^*_{bo}$) represents the mixed strategy of the algorithm (adversary) in the joint mixed strategy $\tau_o^*$: 
 $\forall \tau_p^l \in \tau_p, u_{p}(\tau_o^*) \ge u_{p}(\tau^*_{bo}, \tau_p^l)$, and $ \forall \tau_b^l \in \tau_b, u_b(\tau_o^*) \ge u_b(\tau_b^l, \tau^*_{po})$; 
 or equivalently, 
 $\forall $ pure strategy $ \alpha^l \in \mathbb{A}, u_{p}(\tau_o^*) \ge u_{p}(\tau^*_{bo}, \alpha^l)$, and $\forall $ pure strategy $ \beta^l \in \mathbb{B}, u_b(\tau_o^*) \ge u_b(\beta^l, \tau^*_{po}).$  
It indicates that at NE, neither the algorithm player nor the adversary player can have better utility by unilaterally changing its strategy.
\begin{prop}\label{existense_NE}
There exists a mixed Nash equilibrium in the game played by the algorithm and the adversary.
\end{prop}

\begin{proof}



The existence of the mixed NE is based on Brouwer's Lemma: a continuous function $F(\cdot)$, which maps a non-empty, compact, convex set to the set itself, always has a fixed point $x^*$ such that $x^* = F(x^*)$. Since the probability simplex is non-empty, compact and convex, all we need to show is to find such a $F(\cdot)$.

We first define a function $g_{p}^l(\tau_o) = max(u_{p}(\tau_{bo}, \mathbb{A}_l)-u_{p}(\tau_o), 0)$, where $\tau_{bo}$ means adversary's strategy in $\tau_o$. Given joint mixed strategy $\tau_o$, this function evaluates by fixing the mixed strategy of adversary $\tau_{bo}$, whether the pure strategy $\mathbb{A}_l \in \mathbb{A}$ would be better than the mixed strategy of algorithm in $\tau_o$ for the algorithm. 

Next, we define a continuous function $F_p: \tau_o \rightarrow \tau_o$, with $F_{pl}(\tau_o) = \frac{\tau_{po}^l + g_p^l(\tau_o)}{1+\sum_{j=1}^{|\mathbb{A}|} g_p^j(\tau_o) }, \forall l \in [1,|\mathbb{A}|]$, where 
$\tau_{po}^l$ is algorithm's probability of choosing pure strategy $\mathbb{A}_l$ in its mixed strategy in $\tau_o$. The denominator is to normalize $F_{pl}(\tau_o)$ so that $\sum_{l=1}^{|\mathbb{A}|} F_{pl}(\tau_o) =1$, that is, 
the values of $F_{pl}(\tau_o)$ for all $l$ form a probability distribution. Hence,  $F(\tau_o)$ for all $l$ maps the  non-empty, compact,  convex set to the set itself, with only the algorithm's mixed strategy updated. Based on the Brouwer's Lemma, there exists a fixed point, which is a joint mixed strategy, $\tau_o^*$, 
satisfying
\begin{equation}\label{nash_existence}
\tau_{po}^{*l} =  \frac{\tau_{po}^{*l} + g_p^l(\tau_o^*)}{1+\sum_{j=1}^{|\mathbb{A}|} g_p^j(\tau_o^*) }, \forall l
\end{equation}

There are the following possible cases regarding equation (\ref{nash_existence}): (i)~If $\sum_{l=1}^{|\mathbb{A}|} g_p^l(\tau_o^*)=0$,  {\em i.e.}, 
fixing the adversary's mixed strategy according to the definition of $g_p^l(\cdot)$, varying algorithm's strategy can not obtain larger utility for the algorithm: no pure strategy of the algorithm is better than this mixed strategy; deviating this mixed strategy towards any pure strategy cannot lead to a larger utility. By the definition of Nash Equilibrium, the mixed strategy $\tau_o^*$ in this case is already an NE. (ii) If $\sum_{l=1}^{|\mathbb{A}|} g_p^l(\tau_o^*) > 0$, then for $\forall l, g_p^l(\tau_o^*) >0$, as otherwise, $\tau_{po}^{*l}$ can not be the same. $g_p^l(\tau_o^*) >0$ is equivalent to $u_p(\tau_{bo}^*,\mathbb{A}_l) > u_p(\tau_o^*)$. Taking expectation of both sides of this inequality over $\tau_{po}^*$, we can derive $\sum_{l=1}^{|\mathbb{A}|} \tau_{po}^{*l} u_p(\tau_{bo}^*,\mathbb{A}_l) > \sum_{l=1}^{|\mathbb{A}|} \tau_{po}^{*l} u_p(\tau_o^*)$. For the right hand side of this new inequality, we have $ \sum_{l=1}^{|\mathbb{A}|} \tau_{po}^{*l} u_p(\tau_o^*) =u_p(\tau_o^*)\sum_{l=1}^{|\mathbb{A}|} \tau_{po}^{*l} = u_p(\tau_o^*) $, while for the left hand side, we can get $\sum_{l=1}^{|\mathbb{A}|} \tau_{po}^{*l} u_p(\tau_{bo}^*,\mathbb{A}_l) = u_p(\tau_o^*)$ by definition. Combining both sides together, we reach a contradiction $u_p(\tau_o^*) > u_p(\tau_o^*)$; so this case is impossible. Therefore, we are only left with case (i), which indicates NE. Similar discussion can be applied to the adversary side.
\end{proof}


\subsection{Convergence}
\label{sec:convergence}

\subsubsection{Convergence proof}
We next show that the strategies of the algorithm and the adversary converge to NE. We follow similar idea in \cite{goodfellow2014generative}.
\begin{prop}
If 
in each round $t$ of training in Alg.~2, the adversary is allowed to reach its 
optimal strategy given the current strategy of the algorithm, 
 the algorithm NN converges to the Nash equilibrium strategy. At the same time, the optimal strategy of the adversary given the NE strategy of the algorithm is also the adversary's NE strategy.  
\end{prop}


\begin{proof}
Given the strategy of the adversary, the algorithm's problem (\ref{obj:maximization_conditioned_welfare}) is equivalent to problem (\ref{equi_obj}), which is a convex function of the algorithm strategy. According to \cite{goodfellow2014generative}, we know that for a function $g(x) = sup_{\alpha \in \mathcal{A}}g_{\alpha}(x)$ with $g_{\alpha}(x)$ convex in x for every $\alpha$, we have $\partial g_{\alpha^*}(x) \in \partial g(x) $ if $\alpha^* = arg sup_{\alpha \in \mathcal{A}}g_{\alpha}(x)$. That is, the gradient of $g_{\alpha}(x)$ on x when $\alpha$ is optimally calculated based on the value of $x$ is equivalent to the gradient of function $sup_{\alpha \in \mathcal{A}}g_{\alpha}(x)$. In our case, we can solve for the NE by calculating the gradient on the algorithm strategy when the adversary strategy is optimally calculated based on the algorithm strategy. Due to the convexity of objective (\ref{equi_obj}) for computing the algorithm strategy, sufficiently small gradient descent updates of the algorithm strategy can lead to convergence to the optimal algorithm strategy.

The training process will not end until both the algorithm and the adversary converge to NE strategy, as otherwise, one of them would keep changing its strategy to get better utility and the updating process would continue.
\end{proof}

To allow the adversary to reach its optimal strategy given the current strategy of the algorithm, in each training iteration, $\xi$ would be very large in theory and the gradient of adversary's objective function on $P_i(b_i^{(j)})$ should be computed exactly by enumerating all possible budget sequences in optimization (\ref{obj:maximization_conditioned_gap}). Besides, when a neural network is used, we are optimizing $\theta$ instead of the strategy itself, as in the proof of convergence, and we may not ensure that the adversary is allowed to reach its optimal strategy. However, we will show in our empirical studies even when using $\xi=1$ and our heuristic approach in computing the gradient (see Sec.~\ref{sec: adversary}), the algorithm and adversary NNs converge close to NE.

\subsection{Calculation of Nash equilibrium Strategies} \label{ne_compu}

We first describe how to compute a Nash equilibrium strategy in our zero-sum game, given mixed strategy $\tau_p$ for the algorithm, $\tau_b$ for the adversary and the payoff matrix $C$ with dimension $\mathbb{B} \times \mathbb{A}$, where each entry $(i,j)$ corresponds to the gap achieved by the budgets-prices pair $(\mathbb{B}_i,\mathbb{A}_j)$. If $\tau_p$ is known, the expected gap with pure strategy $\mathbb{B}_i$ of the adversary is the $i^{\mbox{th}}$ element in $C\times\tau_p$ (vector containing expected gaps  corresponding to all possible pure strategies of the adversary). The adversary wants to maximize the gap, so it will play strategies corresponding to the largest value in $C\tau_p$. In order to minimize the gap, the best strategy of the algorithm is to minimize the maximum value in $C\tau_p$: $v_p = min_{\tau_p} max\{(C\tau_p)_1,...,(C\tau_p)_{|\mathbb{B}|}\}$. In contrast, the best response of the adversary is to maximize the minimum value in $\tau_b C$: $v_b = max_{\tau_b} min((\tau_b C)_1,...,(\tau_b C)_{|\mathbb{A}|})$. 
 
The computation of $\tau_p$, NE strategy of the algorithm, can be formulated into the following linear program, where $c_i$ is the $i^{\mbox{th}}$ row of matrix $C$:
\begin{equation}\label{NE_alg}
  Min_{\tau_p} v_p  
\end{equation}
subject to:
\begin{align*}
    \tau_p^j \ge 0, \forall j \in [1,|\mathbb{A}|] &\\
    \tau_p^1+\tau_p^2+\cdots+\tau_p^{|\mathbb{A}|}=1, &  \\
             c_i \tau_p \le v_p, \forall i \in \{1, |\mathbb{B}|\}&
\end{align*}

The computation of $\tau_b$, NE strategy of the adversary, can be formulated into the following linear program, where $c_j$ represents the $j^{\mbox{th}}$ column of $C$:
\begin{equation}\label{NE_adv}
    Max_{\tau_b} v_b
\end{equation}
subject to: 
\begin{align*}
\tau_b^i \ge 0, \forall i \in [1,|\mathbb{B}|] &\\
    \tau_b^1+\tau_b^2+\cdots+\tau_b^{|\mathbb{B}|}=1, &  \\
             \tau_b c_j \ge v_b, \forall j \in \{1, |\mathbb{A}|\}&
\end{align*}
 
Linear programs (\ref{NE_alg}) and (\ref{NE_adv}) are duals of each other. This can easily be verified by the Lagrangian function: $Max_{\tau_b,v_b}L = v_b + \sum_{j=1}^{|\mathbb{A}|} \tau_p^j (\tau_b c_j - v_b) + v_p (1-\tau_b^1-..-\tau_b^{|\mathbb{B}|}) + \sum_{i=1}^{|\mathbb{B}|} z_i \tau_b^i$.
 The gap at the Nash equilibrium is $v^* = v_p=v_b$ given strong duality of linear programs. 
  We can see that solving the Nash equilibrium strategies of our zero-sum game is equivalent to solving two linear programs. 
 
 Next, we discuss the gap and strategies at Nash equilibrium of our system in two cases. 
 
\subsubsection{Resource number no smaller than job number ($R \ge N$)} \label{res>n}

In this case, pricing/budgeting strategies for different users are the same because resource is always abundant upon arrival of each user. So we only need to focus on one user. For any user, the pure strategies for the algorithm and the adversary are sets $A$ and $B$, respectively; $\tau_p$ and $\tau_b$ denote the mixed strategy over set $A$ and $B$, respectively, and $p_l$ and $b_l$ are the probability of choosing a respective pure strategy.
The payoff matrix $C$ of a single user is as follows (since resource is enough to accept the user):

{\centering
\[
C_{m \times n} = 
\begin{bmatrix} 
   c_{ij} = b_i \vmathbb{1}(b_i < p_j)
    \end{bmatrix}
_{\forall i \in \{1,m\}, j \in \{1,n\}}
\]
}

We can derive NE value using (\ref{NE_alg}) (or (\ref{NE_adv})).


If the smallest price from set $A$ is no larger than any budget in set $B$, one of the columns of $C$ is a zero vector (because $\vmathbb{1}(b_i < p_j)$ will always be zero if $p_j$ is the smallest price), and we can get $z^*=0$, that is, the optimal strategy of the algorithm is to accept all users with the smallest price. 

\subsubsection{Resource number smaller than job number ($R < N$)}

The principle of calculating Nash equilibrium in this case is the same as in the previous case, by replacing the pure strategy set from $A$ to $\mathbb{A}$ and $B$ to $\mathbb{B}$ because we can not calculate Nash equilibrium by a single slot now. $\tau_p$ and $\tau_b$ denote the mixed strategy over set $\mathbb{A}$ and $\mathbb{B}$; $\alpha_l$ and $\beta_l$ are pure strategy $l$ in each set, respectively. The payoff matrix $C$ is: \\
\begin{center}
\[
C_{|\mathbb{B}| \times |\mathbb{A}|} = 
\begin{bmatrix} 
   c_{ij}  = Benchmark(\beta_i)-Alg(\beta_i,\alpha_j), 
    \end{bmatrix}_{\forall  i \in \{1,|\mathbb{B}|\}, j \in \{1,|\mathbb{A}|\}} 
\]
\end{center}

We note that the number of pure strategies of both the algorithm and the adversary is increasing exponentially with the length of user sequence, causing computation complexity in solving the corresponding linear program due to enormous numbers of decision variables and constraints.
\section{Empirical Studies}
Through empirical studies, we seek to further answer two questions: (1) Are our algorithm and adversary NN update methods effective in converging to NE? (2) How is the performance of our learned online algorithm compared to state-of-the-art 
online algorithms on the same problem? 

The algorithm neural network consists of three parts: history encoding layer, fully connected layers to process input from the encoding layer, and the output layer. The activation function of the first two parts is Leaky ReLU, and Softmax is used as the output layer function. The adversary neural network is composed of fully connected layers and an output layer. Activation function for the fully connected layers is Leaky ReLU, and Softmax is the output layer function.

In our  experiments, the algorithm NN has 3 fully connected layers following history encoding and input concatenation, and then output layer follows; the adversary NN has 4 fully connected layers before the output layer. The number of neurons in each layer is adjusted accordingly under different experimental settings, within the range of [30, 80]. $\xi$ in Algorithm \ref{alg_train} is set to 1 (same as in \cite{goodfellow2016nips}).  One episode 
indicates training with one complete price/budget sequence. 

\subsection{NN Training Effectiveness}
We evaluate our training of algorithm neural network and adversary neural network separately by replacing the other NN by a multiplicative weight (MW) updated player and compare the strategy learned by the respective NN and the NE strategy calculated with standard game theory methods. MW is a commonly used algorithm in playing zero-sum games to approximately converge to the NE \cite{arora2012multiplicative}, which maintains weights over all pure strategies and the probability of choosing 
each pure strategy is proportional to the corresponding weight. In the training process, weights are updated as follows:
$w^{t+1}(a) = w^t(a)*(1+\eta r^t(a))$, 
where $\eta$ is the learning rate (set to 0.01 in our experiments), $r^t(a)$ is the utility of pure strategy $a$ in updating iteration $t$ (normalized between 0 and 1 to limit the change of weights). We set $s=10$ (Algorithm \ref{alg_train}) in this set of experiments.
\subsubsection{Algorithm NN}
Consider a scenario where the total number of units of resources is 5; both the price set and the budget set are $ \{1,2,3,4,5\}$. The adversary has $25$ pure strategies, each including the first $i$ slots in the sequence $seq = [1,1,1,1,1,2,2,2,2,2,3,3,3,3,3,4,4, \\ 4,4,4,5,5,5,5,5]$, $i \in [1,25]$. Such pure adversary strategies represent worst-case user request sequences in standard online algorithm analysis \cite{zhang2017optimal}: different sequence lengths lead to different final resource utilization levels; increasing budgets make online algorithm perform worst as compared to the offline optimum, as the online algorithm may well exhaust resources by allocating them to earlier users with small budgets, while the offline optimum allocates resources to latecomers. In deciding the mixed strategy of the adversary, $r^t(a)$ represents the expected gap of budget sequence $a$. 

 Given the adversary strategies, we can derive the NE of the game by formulating the zero-sum game into a linear program~(\ref{alg_test}) and solve it using standard method ({\em e.g.}, an interior point method), and obtain that the expected gap and the probability of accepting each user at Nash equilibrium is 7.834 and $[0,0,0,0,0,0,0,0.083, 0.5,0.5,\\ 0.333,0.333, 0.333, 0.333, 0.333,0.25,0.25,0.25,0.25, 0.25,0.2,0.2,0.2,\\0.2,0.2]$, respectively. Note that this strategy is not the only NE strategy of the algorithm: for example, a  probability of 0.75 for accepting the ninth user and a probability of 0.25 for accepting the tenth user also constitute a NE strategy, as long as the gap between the algorithm strategy and each pure strategy of the adversary is no larger than 7.834.
The formulation of effectiveness test of algorithm NN is as follows:
 The problem formulation is similar to (\ref{NE_alg}). Instead of computing the probabilities of choosing different prices, we let decision variables in the LP be the probabilities of accepting different users so that the number of decision variables is only 25 (we assume the algorithm knows the maximum length of user request sequence only when computing this NE). $P_i$ denotes the probability of accepting user $i$. The linear program is as follows:

 \begin{equation}\label{alg_test}
     \min_{\substack{P_i, i \in [1,25]}} z
 \end{equation}
subject to:
\begin{align*}
    P_1+P_2+\cdots+P_{25} \le 5,~ (\ref{alg_test}a)  \\
    0 \le P_i \le 1, \forall i \in [1,25]~(\ref{alg_test}b)\\
    \sum_{i=1}^j seq[i]P_i \le z, \forall j \in [1,25]~(\ref{alg_test}c)
\end{align*}

\noindent(\ref{alg_test}a) is the resource constraint. 
(\ref{alg_test}c) corresponds to all pure strategies of the adversary, where $seq[i]$ denotes the $i^{\mbox{th}}$ budget in the budget sequence. We can then solve this linear program to get NE.
 The training curve of the algorithm neural network is shown in Fig.~\ref{fig:algonly}. 
 We can see that the time-averaged gap (averaged over last 500 episodes) converges to around 8 after 40k episodes of training, very close to the computed gap at NE, showing effectiveness of our update method of training the algorithm NN in ensuring convergence to NE. We note that in this experiment, even though the size of the price set is only 5, when there are 25 users, the number of pure strategies of algorithm is as many as $5^{25}$. 
 
 We can then compute the probability of accepting each user from the NN output (averaged over last 9k episode): 
 $[0, 0, 0, 0, 0, 0.226, 0.222, \\ 0.213, 0.173, 0.115, 0.523,  0. 36,  0.324, 0.212, 0.153, 0.534, 0.352, 0.226, \\0.203, 0.113,  0.345, 0.221,  0.209, 0.208, 0.091]$.   
The largest gap resulted from the above strategy under all pure strategies of the adversary is within 0.5 from the gap value at NE ({\em i.e.}, 7.835), showing closeness of our strategy with the computed NE strategy of the algorithm.  The running time for completing 100k episodes of algorithm NN training is less than 15 minutes.

\subsubsection{Adversary NN}
When testing the effectiveness of adversary neural network, the mixed strategy of the algorithm is updated by MW. 
There are an exponential number of decision variables in the linear program (LP) ( (\ref{NE_adv}) in Section \ref{ne_compu}
) to solve for the adversary's NE strategy using standard game theory method (as its probability distribution is over all combinations of $N$ element from the budget set).
To make this problem computable (without causing out-of-memory error when solving the LP using Python linprog on a server with 64 GB memory), we set the sequence length to 7, price/budget set as $A = \{1,2,3\}$, and the total number of resource units to 3. The pure strategies of the algorithm are the following 3 price sequences: $ seqs = [1,1,2,2,3,3,3],[1,1,1,2,2,2,3]$, $ [1,2,2,2,3,3,3]$, which are chosen following pricing function design in existing online algorithms \cite{zhang2017optimal}, which posts lower prices at the start when resources are abundant so that more users can be accepted while raising the prices with the consumption of resources. 
  In computing mixed strategy of the algorithm using MW, $r^t(a)$ is the additive inverse of the expected gap for choosing price sequence $a$. 

By solving LP (7) 
in our concrete setting, we obtain that the expected gap is 4.333 and the mixed strategy of the adversary is to choose among budget sequences $[1,1,1,1,2,2,2],[1,1,2,2,3,3,3],[1,\\ 2,1,2,3,3,3]$ with a $\frac{1}{3}$ probability each, at Nash equilibrium. Note that this NE strategy of the adversary is not unique, any strategy leading to a gap between this strategy and any pure strategy of the algorithm no smaller than 4.333 is an NE strategy.

Fig.~\ref{fig:advonly} shows the training curve of the adversary NN,  which took about 8 minutes till convergence. We can see that the time-averaged gap is around 4.33 after 60k episodes of training, close to the NE gap value, exhibiting the NE learning ability of our update method for the adversary NN. The strategy learned by our adversary NN is to choose budget 1 with probability of 1 for the first 4 users and budget 2 with probability of 1 for the last 3 users, which corresponds to the first budget sequence. 
Though different from the NE strategy computed above, only choosing the first sequence is also an NE strategy since its expected gap is no less than the NE gap value, given unilateral strategy change of the algorithm. 

\begin{figure*}[t]
	\captionsetup{width=0.3\textwidth}
	\vspace*{-3mm}
	\begin{center}
		\begin{minipage}{0.32\linewidth}
			\includegraphics[width=\textwidth]{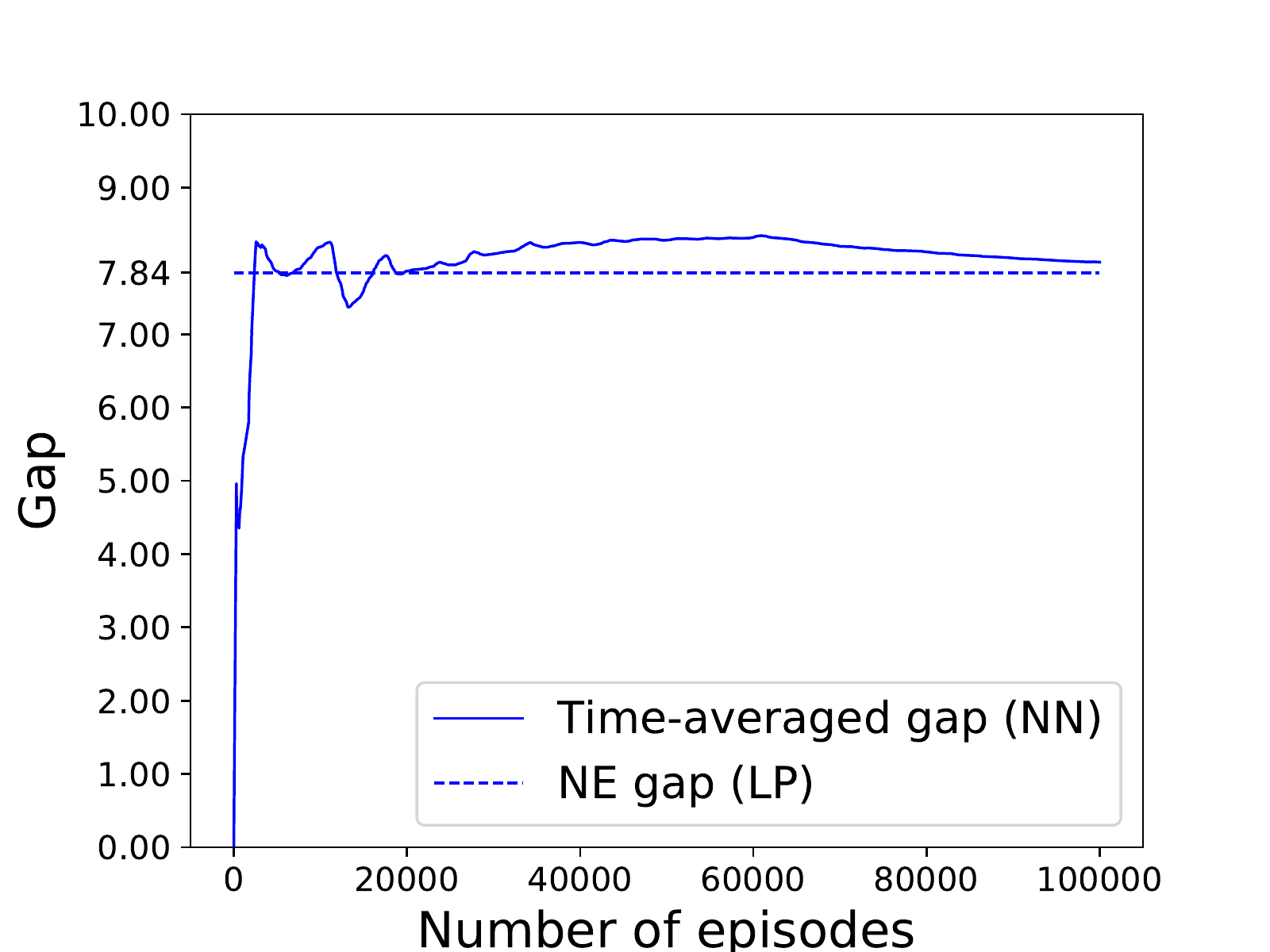}
			\vspace*{-2mm}
            \caption{Algorithm NN training process}
            \label{fig:algonly}
		\end{minipage}
		\begin{minipage}{0.32\linewidth}
			\includegraphics[width=\textwidth]{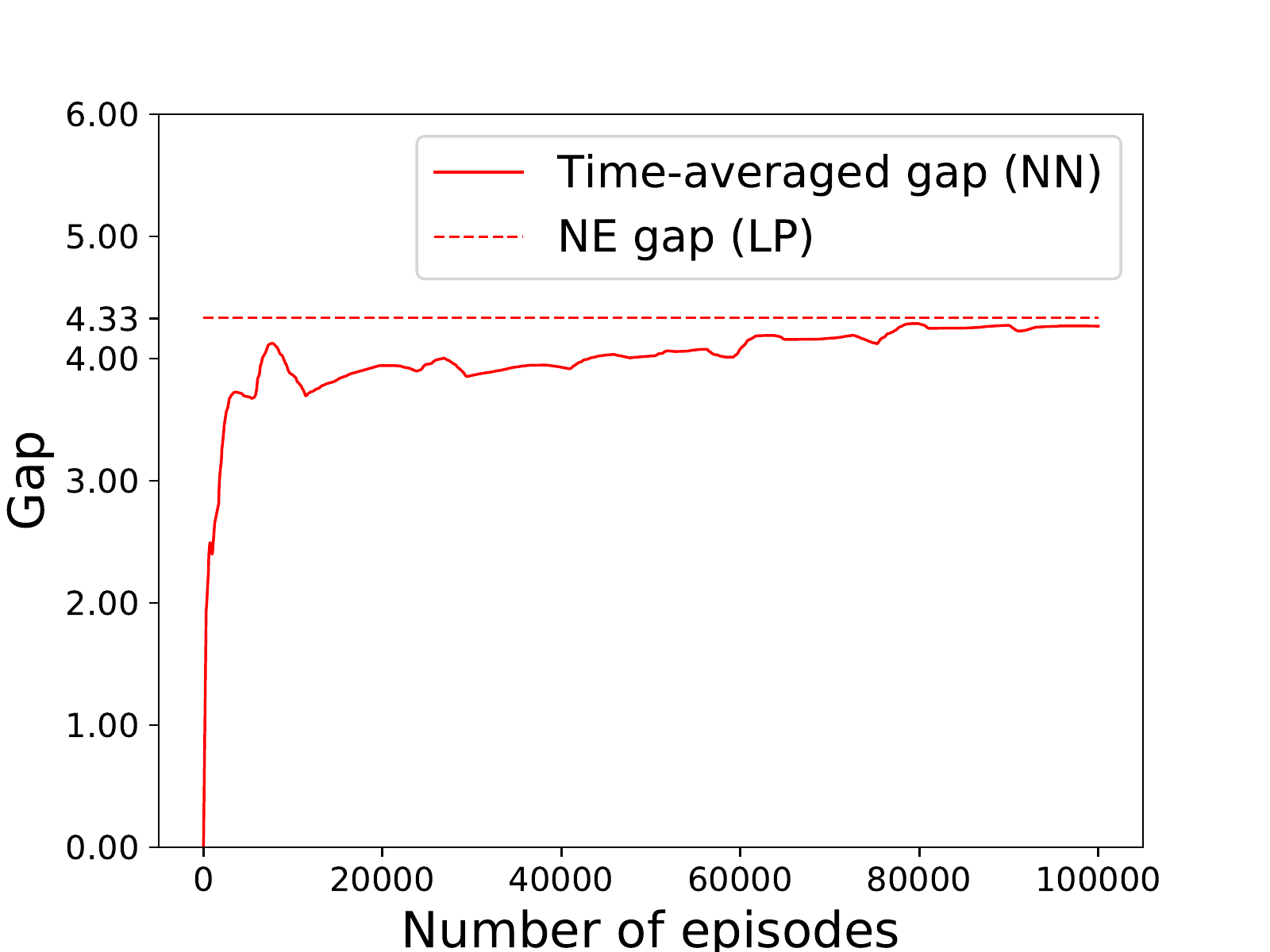}
				\vspace*{-2mm}
            \caption{Adversary NN training process}
            \label{fig:advonly}
		\end{minipage}
		\begin{minipage}{0.32\linewidth}
			\includegraphics[width=\textwidth]{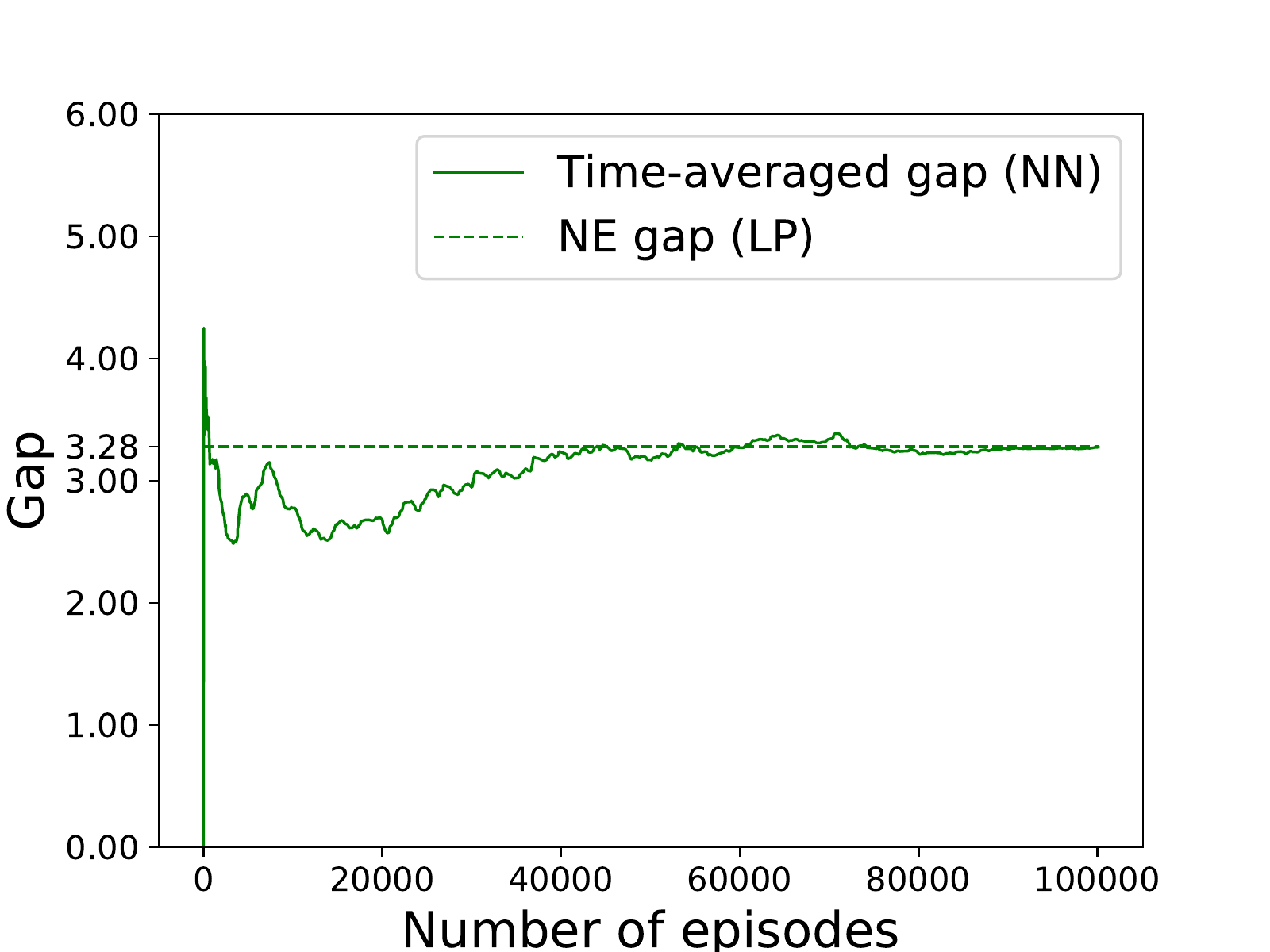}
			\vspace*{-2mm}
    \caption{Two NNs training process}
    \label{fig:twonn}
		\end{minipage}
	\end{center}
	\begin{center}
		\begin{minipage}{0.32\linewidth}
			\includegraphics[width=\textwidth]{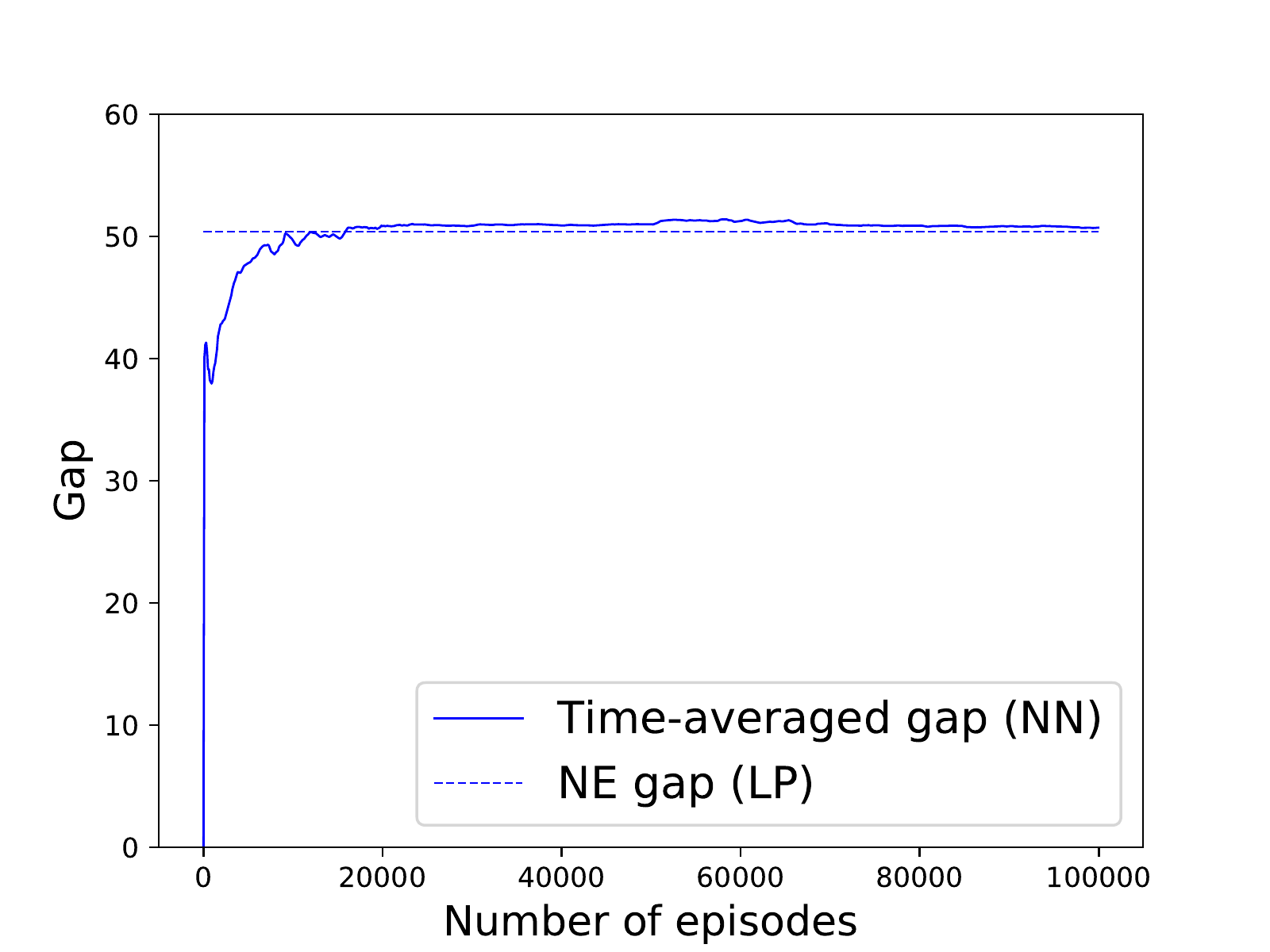}
			\vspace*{-2mm}
            \caption{Algorithm NN training process}
            \label{fig:algonly_20_40}
		\end{minipage}
		\begin{minipage}{0.32\linewidth}
			\includegraphics[width=\textwidth]{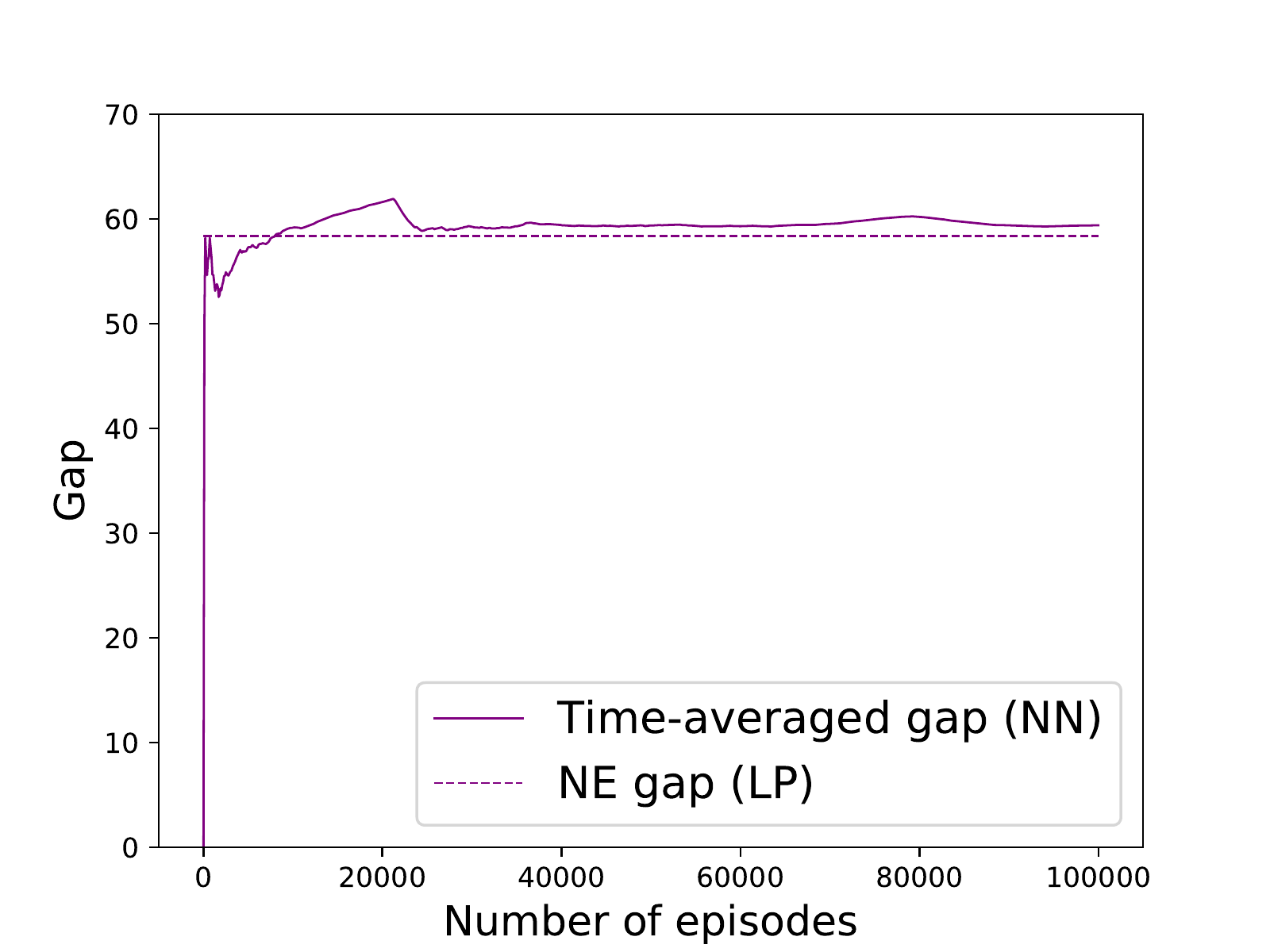}
				\vspace*{-2mm}
            \caption{Algorithm NN training process}
            \label{fig:algonly_20_60}
		\end{minipage}
		\begin{minipage}{0.32\linewidth}
			\includegraphics[width=\textwidth]{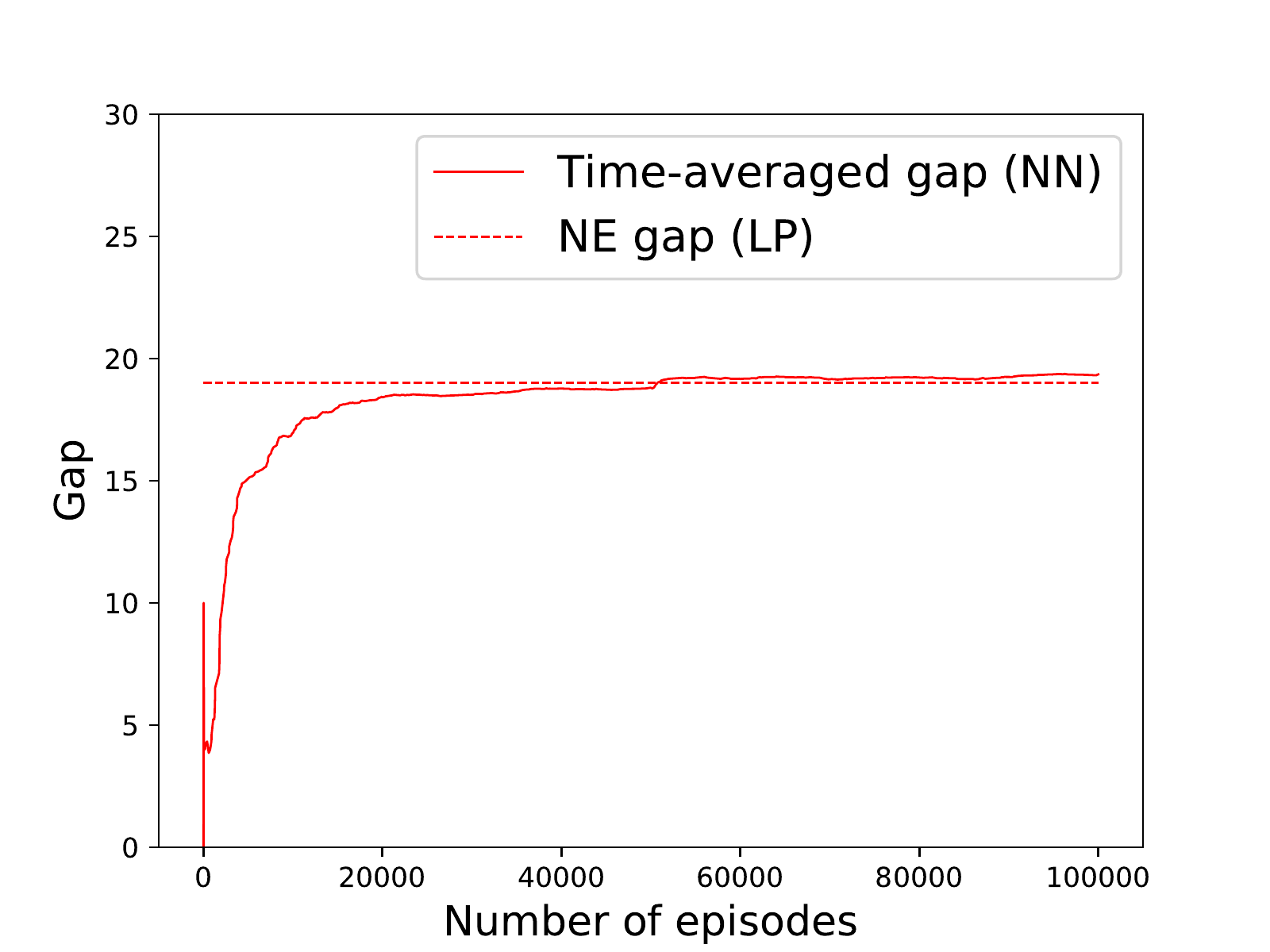}
			\vspace*{-2mm}
    \caption{Adversary NN training process}
    \label{fig:advonly_10_7}
		\end{minipage}
	\end{center}
\end{figure*}

\subsubsection{Training Algorithm and Adversary NNs together}
Next, we show the effectiveness of our training methods when algorithm and adversary NNs are trained together, under the following setting: the budget set is $B = \{2,4,6\}$ and the price set is $A = \{1,3,5,7\}$; maximal sequence length is 7 and total resource number is 3; the pure strategy set of the adversary contains all combinations of 7 budgets from set $B$ and the pure strategy set of the algorithm contains all combinations of 7 prices from set $A$.

By solving the corresponding linear programs in (\ref{NE_alg}) and (\ref{NE_adv}), we obtain that the gap at NE is 3.279. The expected gap values throughout training of the algorithm and the adversary NNs are shown in Fig.~\ref{fig:twonn}, which are calculated using both algorithm and adversary strategies learned. We observe that after about 80k episodes of training, the expected gap approaches the computed gap at NE, further exhibiting the effectiveness of our training methods.  The overall running time for 100k-episode training of two NNs is about 20 minutes.

{
\subsubsection{Training efficiency} 
To better illustrate the convergence behaviour of our method, we further test the algorithm NN and the adversary NN with larger price/budget set and longer sequences. Fig.~\ref{fig:algonly_20_40} and Fig.~\ref{fig:algonly_20_60} present the convergence results of the algorithm NN when the adversary is updated by MWU. The price and budget can be any integers between 1 and 20. The budget sequence for Fig.~\ref{fig:algonly_20_40} is constructed by repeating each 
budget choice from 1 to 20 twice (sequence length = 40), while the budget sequence for Fig.~\ref{fig:algonly_20_60} is constructed in a similar manner by repeating each value in the budget set from low to high for three times (sequence length = 60). The total resource number is set to 10 for both budget sequences. The NE values computed by solving similar linear programs as (\ref{alg_test}) for these two budget sequences are 50.39 and 58.39, respectively. From the training results we can see that given larger budget/price sets and longer sequences, the update method for the algorithm NN can still ensure the convergence to NE. The running times for 100,000 episodes for Fig.~\ref{fig:algonly_20_40} and Fig.~\ref{fig:algonly_20_60} are about 20 minutes and 30 minutes, respectively. As for the adversary NN, we extend its price and budget sets to integers between 1 and 10 (LP with larger sets would be unsolvable due to memory error). The sequence length is 7 and the resource number is 3. Pure strategies/price sequences of the algorithm (updated by MWU) are as follows: $[[1,2,3,4,5,6,7], [2,4,4,6,6,9,10],[1,2,2,2,2,2,3],[6,6,7,7,\\ 8,9,10]]$. The linear program is solved on a server with 256GB RDIMM memory 
and the NE value computed is 19. 
Fig.~\ref{fig:advonly_10_7} shows the training convergence for the above described game between the adversary NN and the MWU-based algorithm, where the convergence took about 16 minutes.

Next, we evaluate the time needed for conducting training on NVIDIA GeForce 1080Ti, and compare the time consumed by our method with other online algorithms during the inference stage. We report the time needed for running one episode in the training process, which starts with sequence sampling according to the opponent's strategy and ends with one update step based on the gradient derived from Section \ref{sec:model}. We also report the time needed to make decisions for the same budget sequence with our algorithm NN and with other online algorithms (aka the inference time) in Table \ref{tab:time}. 
The sequence length is fixed to 100. We observe that the training process would well take hours to complete if the sequence length is long. However, training is done offline. 
On the other hand, the average inference time is less than 0.001s with all methods. Even though decision making with our NN-based method is slower than the other methods, the inference time is still small enough. 
\begin{table*}[t]
    \centering
    \begin{tabular}{ c|c||c|c|c|c|c }
 \hline
  \multicolumn{2}{c||}{Training} & \multicolumn{5}{c}{Inference}\\
  \hline
 \thead{Alg NN } & \thead{Adv NN } & \thead{ NN}& \thead{OPT-Online}& \thead{KP-Threshold} & \thead{Randomized} & \thead{Greedy} \\
 \hline
0.067s  & 0.14s   &0.0008s&   0.00007s & 0.00007s& 0.00006s &0.00004s\\
 \hline
\end{tabular}\caption{Average processing time}
    \label{tab:time}
\end{table*}
}
\vspace{-5mm}

\begin{figure*}[t]
	\captionsetup{width=0.3\textwidth}
	\vspace*{-3mm}
	\begin{center}
		\begin{minipage}{0.32\linewidth}
			\includegraphics[width=\textwidth,valign=t]{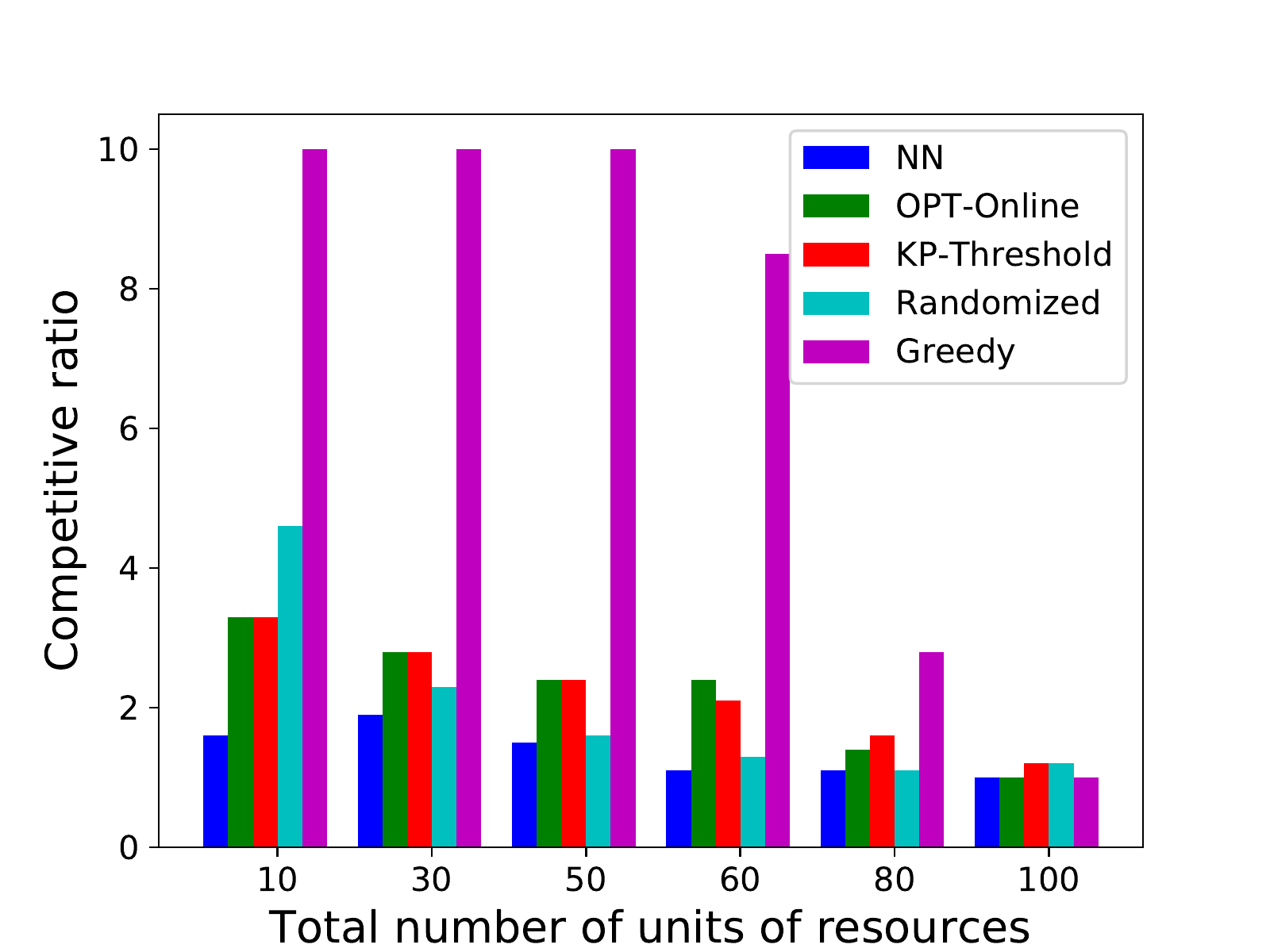}
			\vspace*{-2mm}
			\caption{CR: same sequence length} 
			\label{sim:fig1}
		\end{minipage}
		\begin{minipage}{0.32\linewidth}
			\includegraphics[width=\textwidth,valign=t]{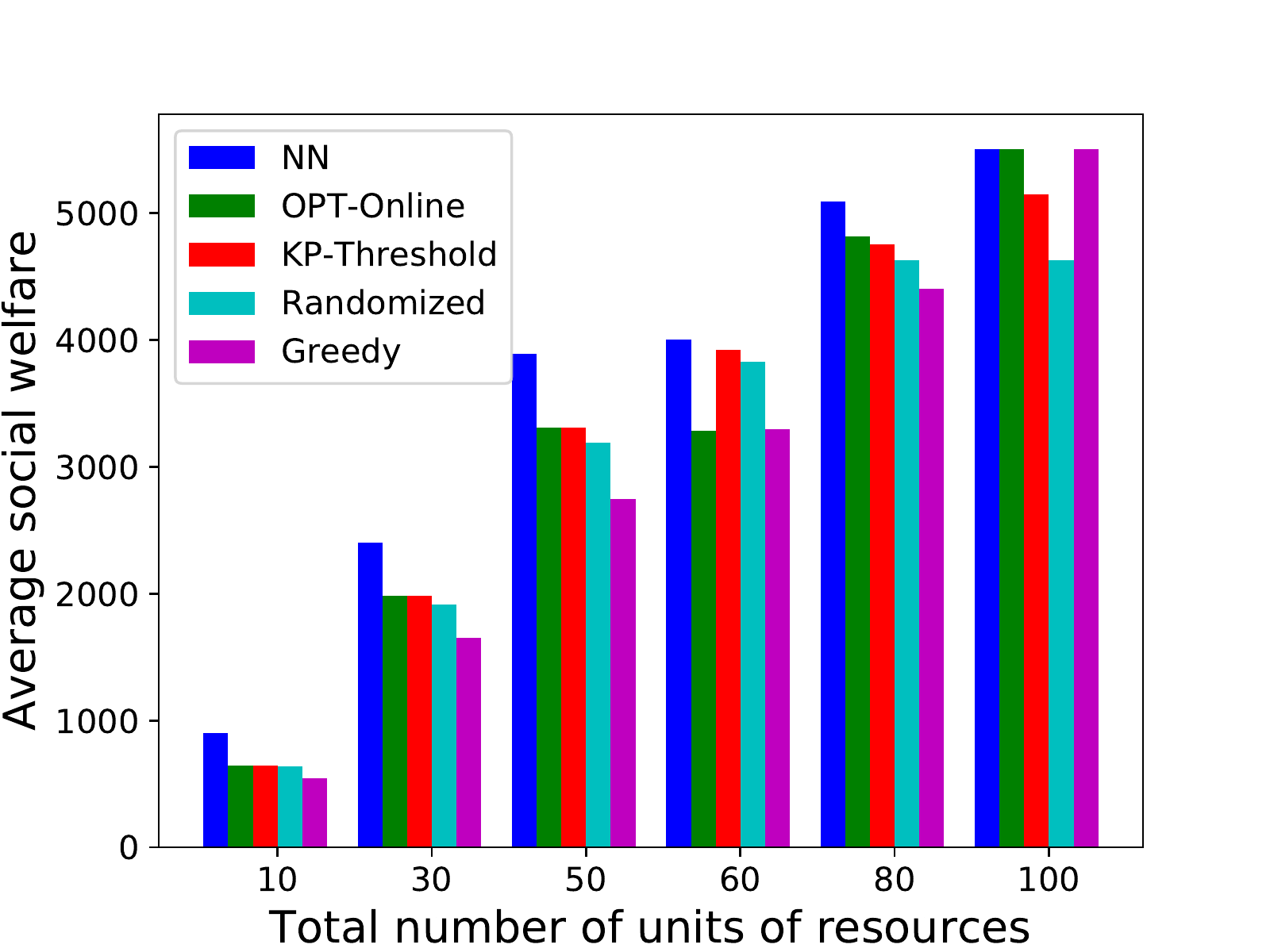}
			\vspace*{-2mm}
			\caption{Social welfare: same sequence length}
			\label{sim:fig2}
		\end{minipage}
		\begin{minipage}{0.32\linewidth}
			\includegraphics[width=\textwidth,valign=t]{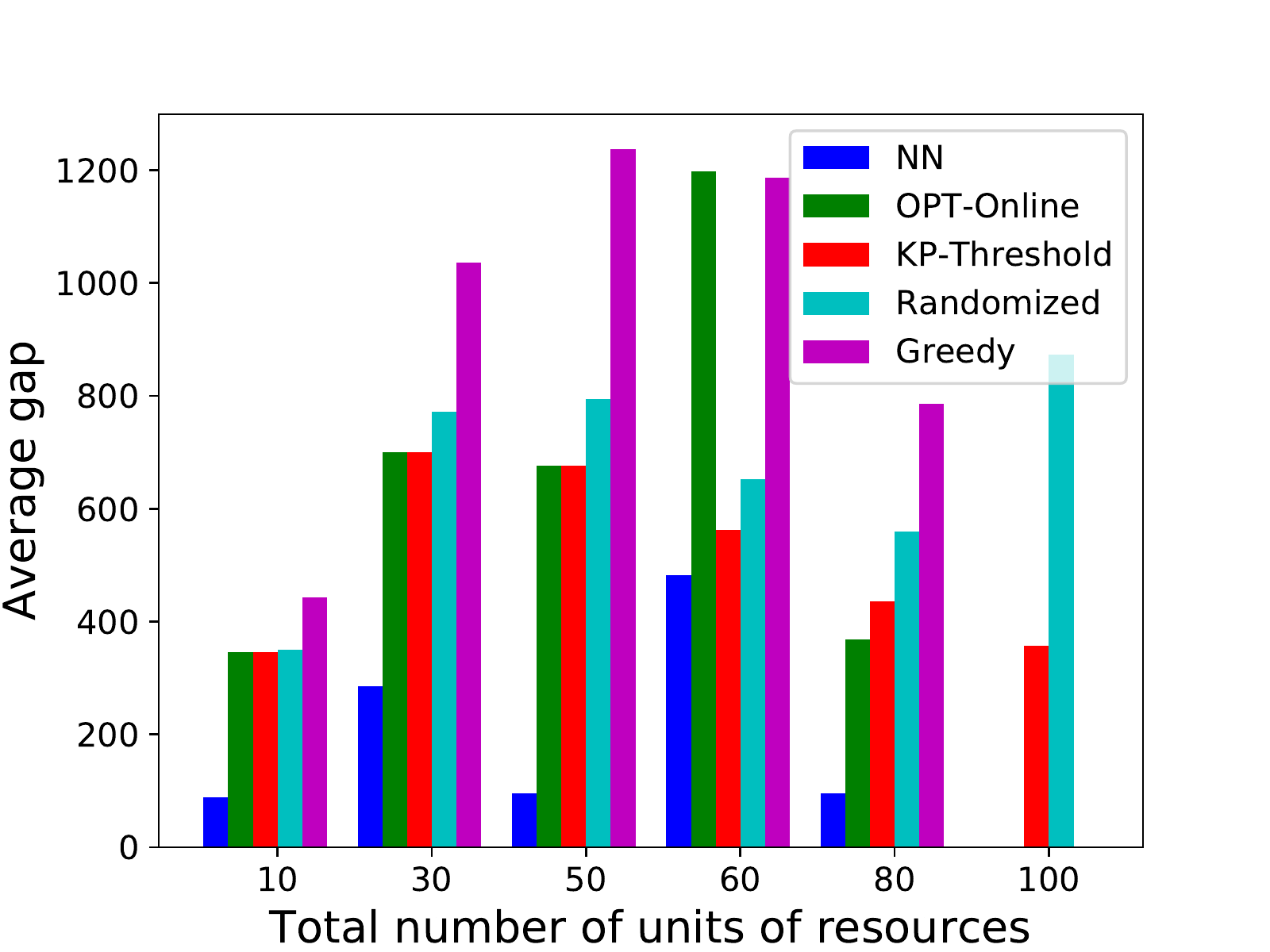}
			\vspace*{-2mm}
			\caption{Gap:same sequence length}
			\label{sim:fig3}
		\end{minipage}
		\end{center}
		\begin{center}
		\begin{minipage}[t]{0.32\linewidth}
			\includegraphics[width=\textwidth]{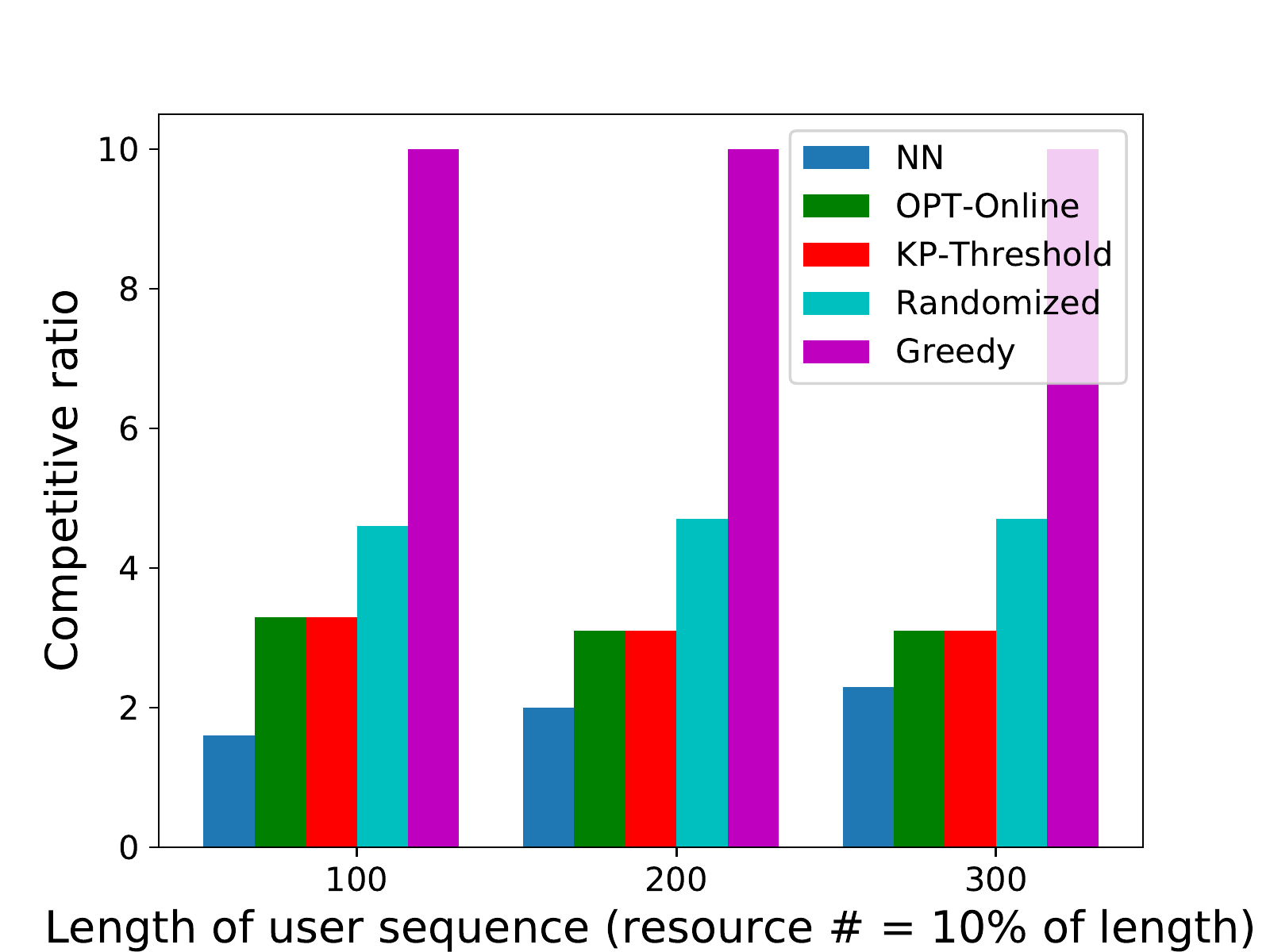}
			\vspace*{-2mm}
			\caption{CR: same supply/demand ratio}
			\label{sim:fig4}
		\end{minipage}
		\begin{minipage}[t]{0.32\linewidth}
			\includegraphics[width=\textwidth]{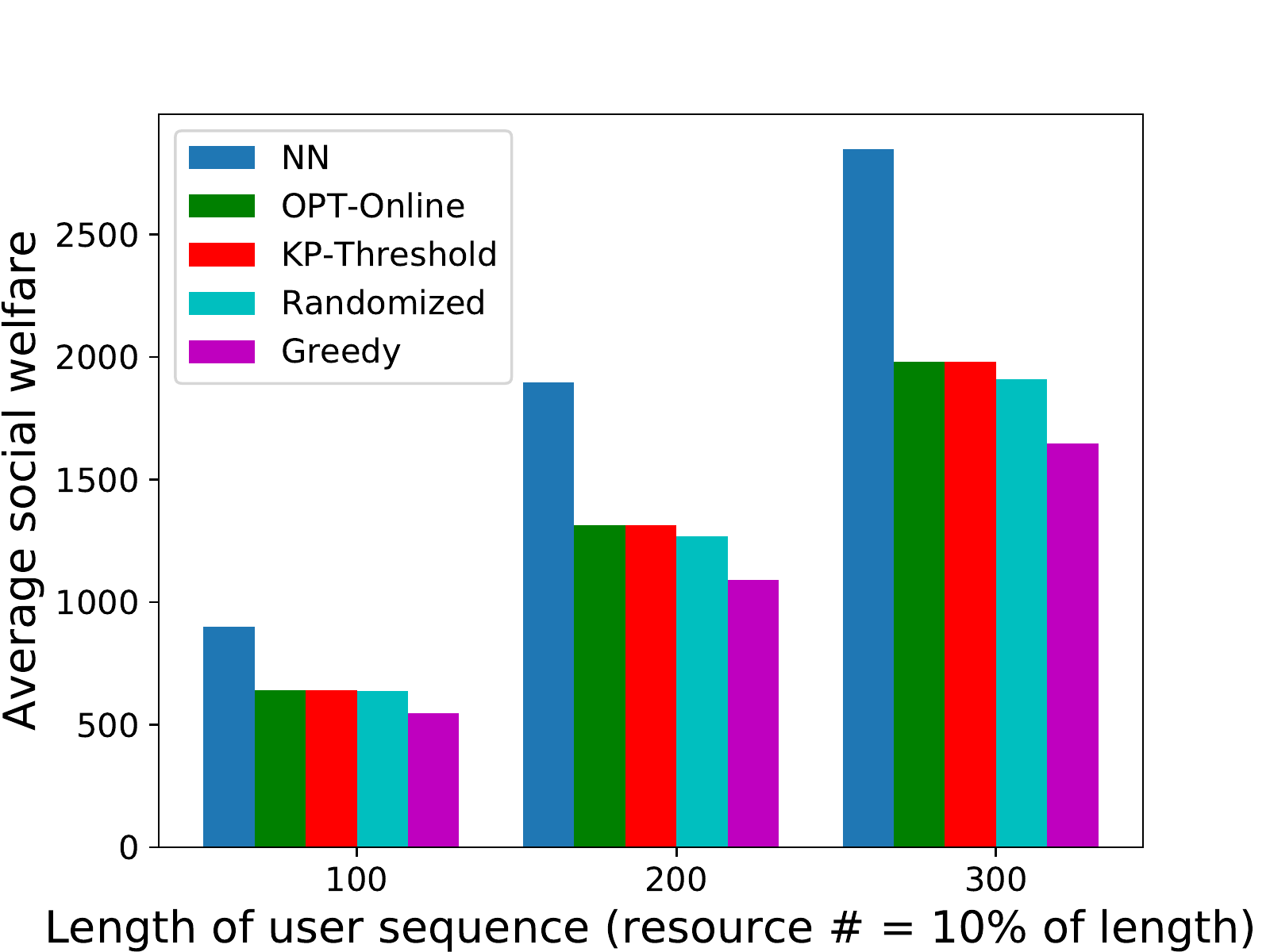}
			\vspace*{-2mm}
			\caption{Social welfare: same supply/demand ratio}
			\label{sim:fig5}
		\end{minipage}
		\begin{minipage}[t]{0.32\linewidth}
		\includegraphics[width=\textwidth]{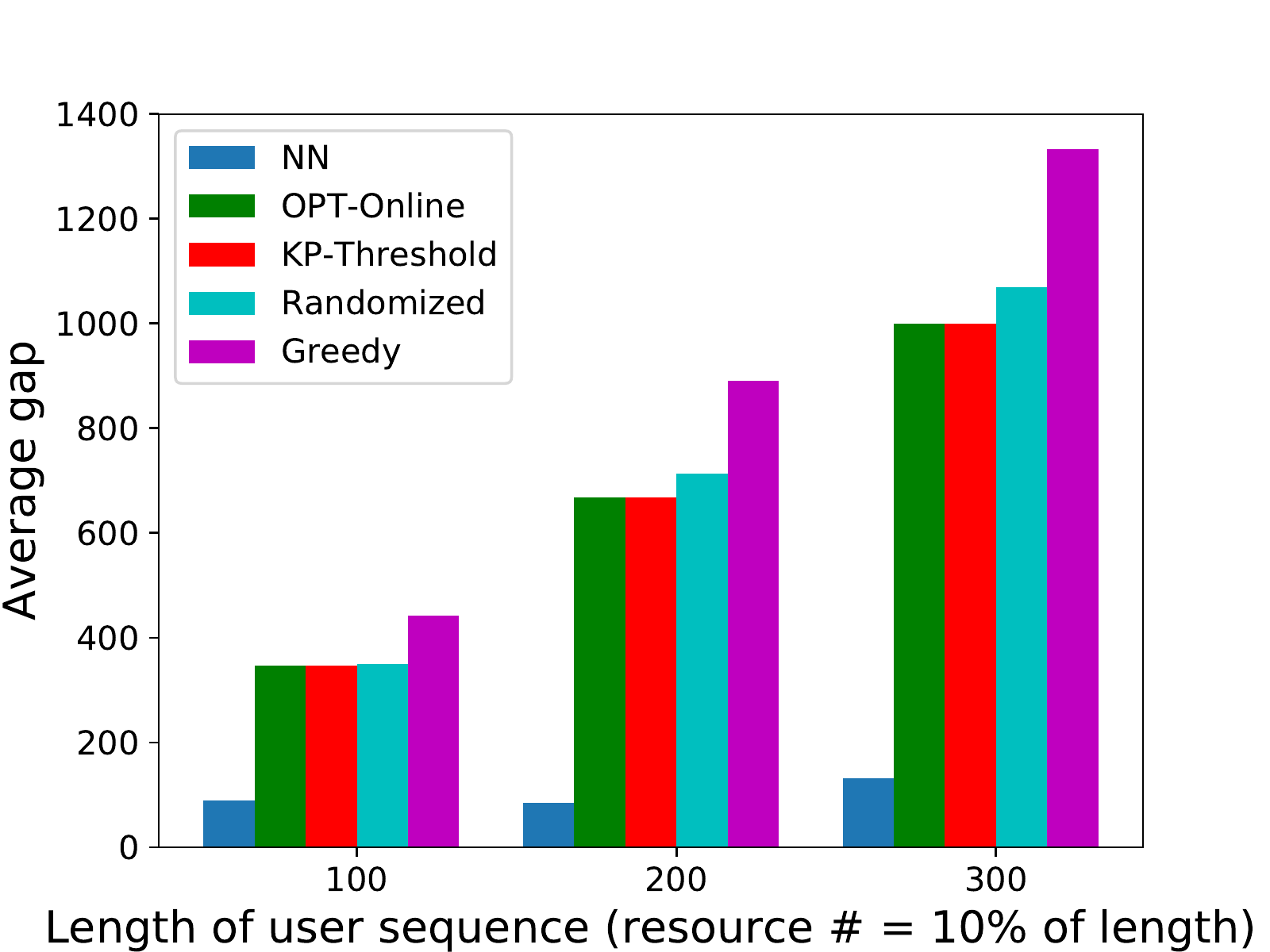}
		\vspace*{-2mm}
		\caption{Gap: same supply/demand ratio}
		\label{sim:fig6}
	\end{minipage}
	\end{center}
	\vspace*{-6mm}
\end{figure*}

\subsection{Comparison with existing online algorithms}

To evaluate the pricing strategy learned by our algorithm NN, we compare its performance with state-of-the-art online algorithms. 
Let $U$ and $L$ be upper bound and lower bound of the user budget per unit of resource, and $z$ denote the fraction of used resource.
The baseline algorithms we compare with are as follows: (1) \textbf{OPT-Online} \cite{zhang2017optimal}, which is a deterministic online algorithm for single-type, non-recycled resource pricing and allocation, obtaining optimal competitive ratio based on two assumptions (see Section~\ref{sec:relatedwork}). Its pricing strategy is based on the relationship between total resource demand and resource supply and uses $U, L, z$; so does its competitive ratio.  The optimality proofs of \textbf{OPT-Online} rely on continuous price/budget while the price/budget set in our setting is discrete. When budgets are discrete, some important properties, such as the largest accepted budget by the offline optimum equals the largest accepted budget by the online algorithm minus $\epsilon$ does not hold any more, which indicates that the gap between offline optimal and online algorithm would be smaller when price/budget set is discrete. We investigate a discrete price set to facilitate the learning algorithm design and leave the continuous price function case as our future work. (2) \textbf{KP-Threshold} \cite{zhou2008budget}, which uses price function $\psi(z)= (Ue/L)^z(L/e)$ ($e$ is the base of natural logarithm) to decide which job to accept. (3) \textbf{Randomized} algorithm (mentioned in \cite{zhou2008budget}), with a competitive ratio of $\mathcal{O}(\log_2 (U/L))$: for each item, this algorithm samples an integer $i$ uniformly randomly from $[0, \log_2(U/L)]$; if the ratio of user budget per unit resource over $L$ is no smaller than $2^i$, the user is accepted. 
 (4) \textbf{Greedy} algorithm, which accepts every arrived user until all resources have been allocated, with a competitive ratio of $U/L$.

In our following evaluation, all results are averaged over 1000 sequences. NN parameters of the last 1k training episodes are saved and uniformly randomly chosen for evaluation. 
The price/budget set is $[10,20,30,40,50,60,70,80, 90,100]$. $s$ in Algorithm \ref{alg_train} is set between 32 and 64. 

We first compare the competitive ratio achieved by our learned algorithm with the benchmark algorithms. 
 Competitive ratio is computed as the ratio of the offline optimum and performance of the online algorithm under the worst-case user request (aka budget) sequence, for each algorithm. For our algorithm, the worst-case budget sequence is provided by the adversary NN. For \textbf{OPT-Online} and \textbf{KP-Threshold}, given the final resource utilization level, we can calculate the largest posted price $p^*$ of the respective algorithm; we set the budget values for accepted user as close as possible to the price that online algorithm posted upon its arrival, 
  so that the objective value of online algorithm can be minimized, and budgets of other users to $p^* - \epsilon$ so that the objective value of offline optimum can be maximized \cite{zhang2017optimal}.  
 For the \textbf{Randomized} online algorithm, the worst-case sequence contains gradually increased budgets. 
 For the \textbf{Greedy} algorithm, the worst-case sequence starts with the lowest budget value while after resources have been used up, highest budget value will follow. 

Fig.~\ref{sim:fig1} shows the competitive ratio when the user sequence length is 100.  The total resource supply varies, representing different relationships between resource demand and supply \cite{zhang2017optimal}. 
We can see that our algorithm always achieves the smallest competitive ratio. 
Especially when the resource is more scarce, our competitive ratio is smaller - almost half of those of other algorithms. 
Fig.~\ref{sim:fig4} compares the competitive ratios when the user sequence length varies while the ratio between total resource supply and total resource demand is fixed to 10\%. Our algorithm outperforms baselines with lower competitive ratios. 

Fig.~\ref{sim:fig2} and Fig.~\ref{sim:fig3} compare the social welfare and gap achieved by the algorithms under uniformly randomly generated budget sequences of length 100. Our algorithm achieves the best social welfare and smallest gap to offline optimum. In the case that resource number is no smaller than the user number, we see the gap achieved by our algorithm is zero, consistent with our analysis in Sec. 7. 

Fig.~\ref{sim:fig5} and Fig.~\ref{sim:fig6} further compare the performance when the user sequence length varies and the ratio between total resource supply and total resource demand is fixed (10\%). Similar observations can be made.

\subsection{Comparison with reinforcement learning}
We next show the superior convergence performance of our method over the reinforcement learning algorithm. The budget and price sets are both $[1,2,3]$ in this experiment, and the resource number is set to 5. For a given budget sequence (each budget is chosen uniformly at random from the budget set), we train two NNs with the same architecture by our update method and the REINFORCE update method, separately. REINFORCE is a commonly used reinforcement learning algorithm for sequential decision making in the deep learning literature \cite{yu2017seqgan}\cite{kool2018attention}, which updates sampled price values by gradient descent according to the average performance gap of complete price sequences, which are formed by sampled price values and following RL-learned strategy for un-sampled slots. In Fig.~\ref{sim:rl1}, the sequence length is 5; the exploration strategy for RL is a simple $\epsilon$-greedy strategy ($\epsilon$ = 0.3), that is, with probability $\epsilon$, the sampled price according to the RL strategy will be replaced by a random value from the price set to encourage exploration. We can see that our method outperforms RL in both solution quality and convergence speed. Next, we equip RL with an improved exploration strategy and keep the sequence length at 5. Instead of replacing a sampled price with a random price from the price set, we replace it with the optimal price with probability 0.3 in the training process, to guide the sample process to better decisions. Fig.~\ref{sim:rl2} plots the training process in this scenario. We can see that RL heavily relies on a good exploration strategy and performs better when the guided exploration strategy is applied, but our method is still better than RL. Lastly, we compare RL plus the same guided exploration with our method with a longer sequence (sequence length = 10). Fig.~\ref{sim:rl3} shows the training process. We can see that when the sequence becomes longer, training convergence becomes harder for RL due to its larger exploration space and the difficulty to break the correlation between steps.

\begin{figure*}[t]
	\captionsetup{width=0.3\textwidth}
	\vspace*{-3mm}
	\begin{center}
		\begin{minipage}[t]{0.32\linewidth}
			\includegraphics[width=\textwidth]{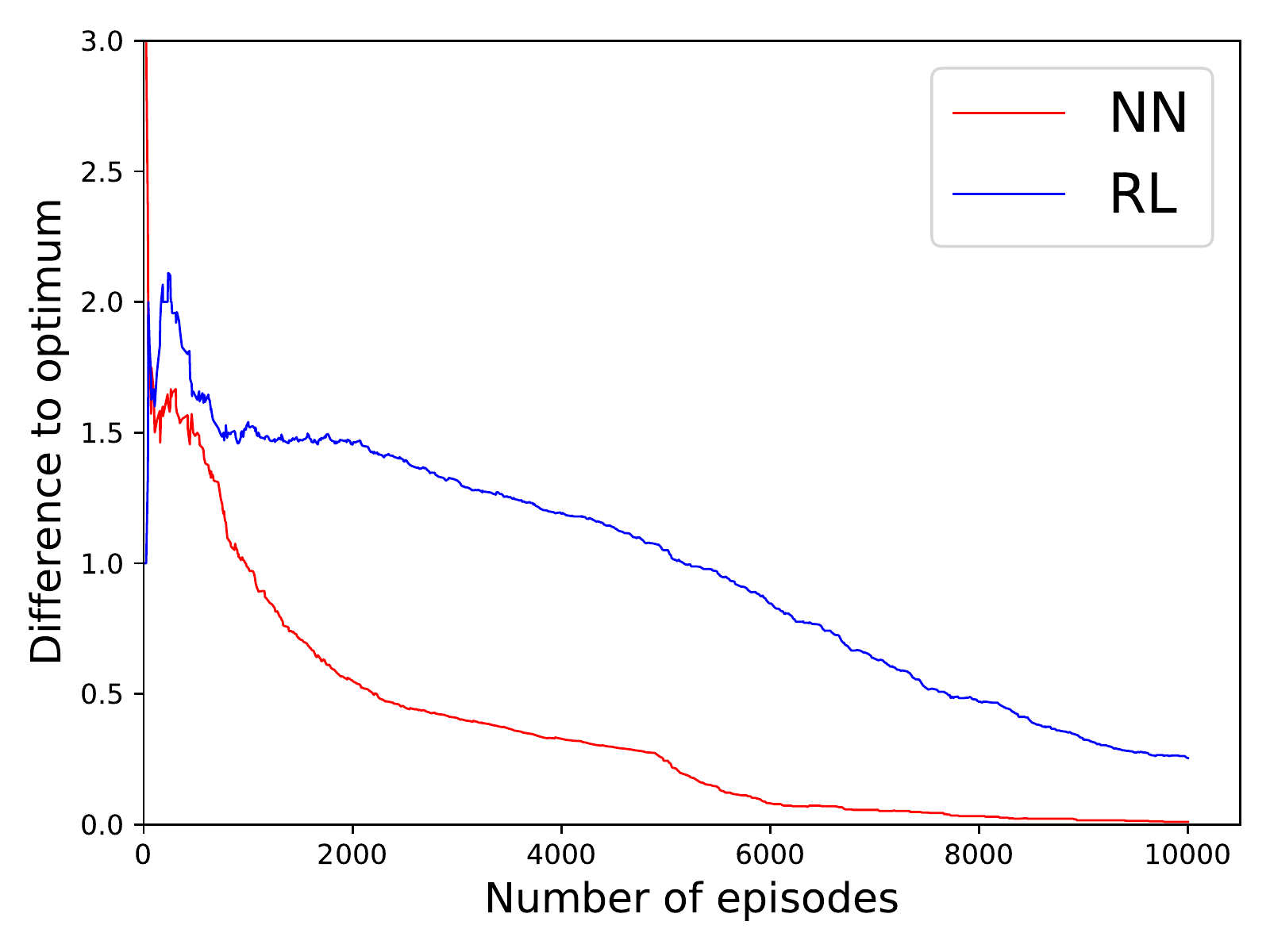}
			\vspace*{-2mm}
			\caption{$\epsilon$-greedy exploration} 
			\label{sim:rl1}
		\end{minipage}
		\begin{minipage}[t]{0.32\linewidth}
			\includegraphics[width=\textwidth]{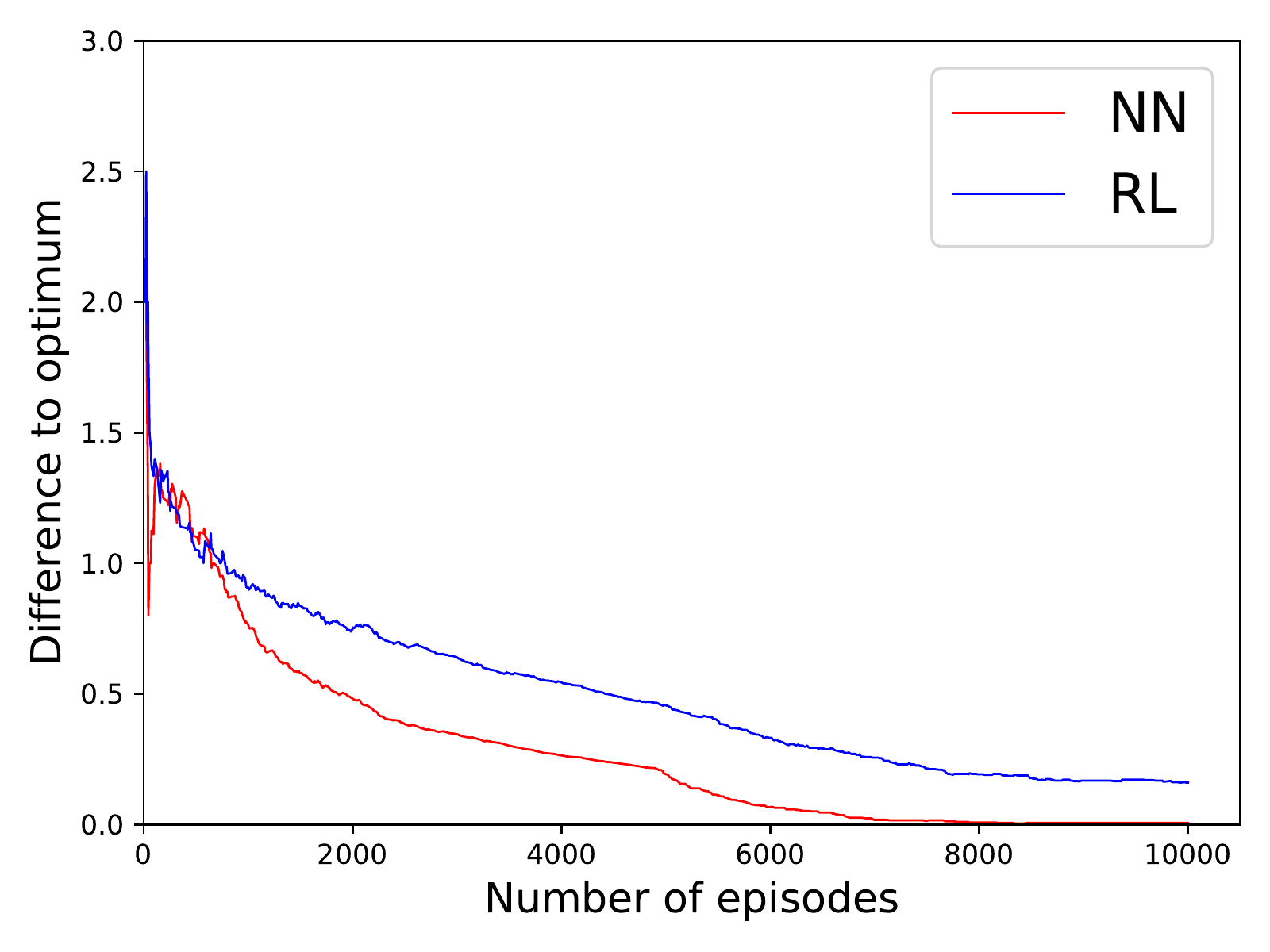}
			\vspace*{-2mm}
			\caption{Guided exploration with sequence length = 5}
			\label{sim:rl2}
		\end{minipage}
		\begin{minipage}[t]{0.32\linewidth}
			\includegraphics[width=\textwidth]{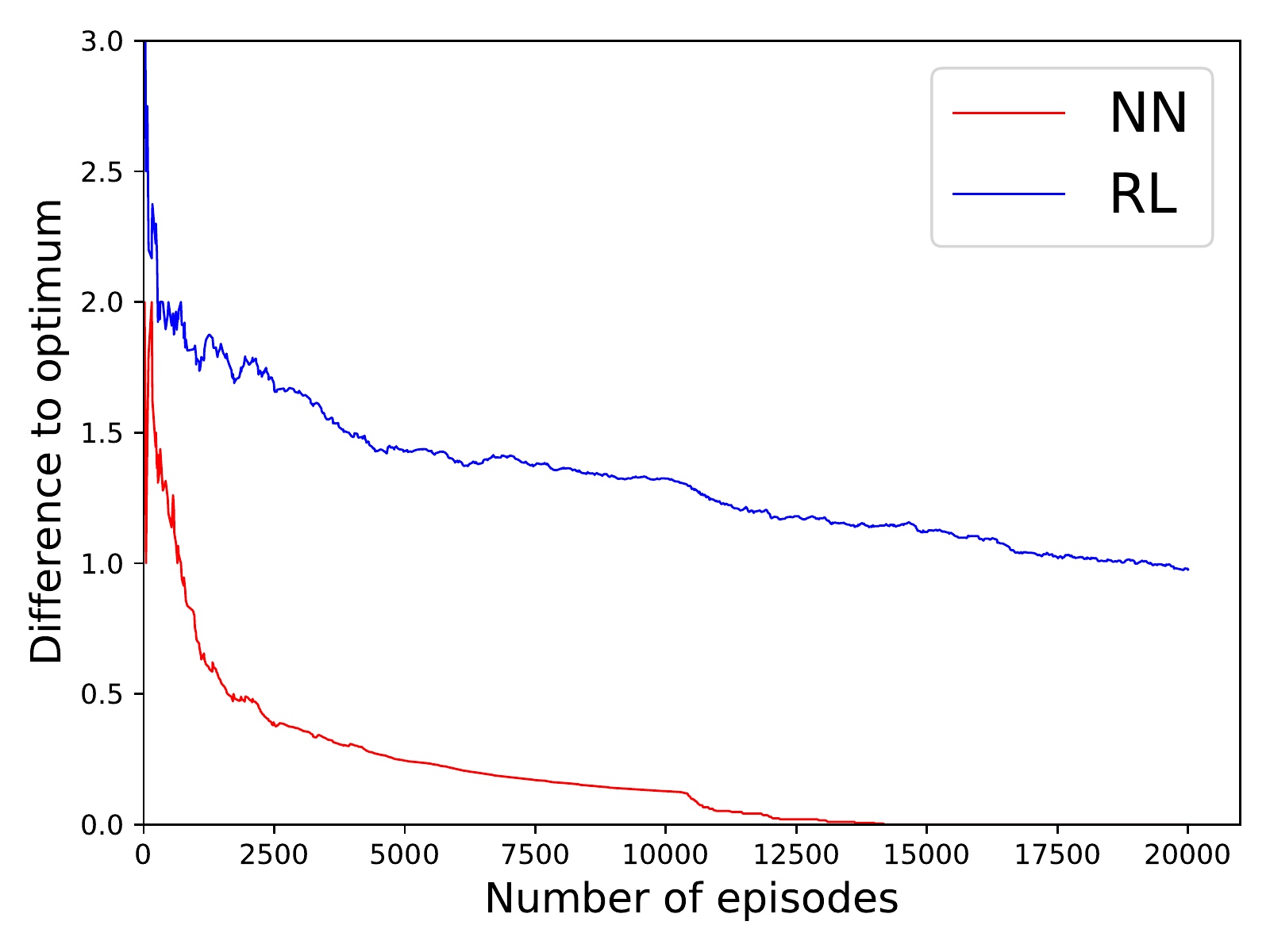}
			\vspace*{-2mm}
			\caption{Guided exploration with sequence length = 10}
			\label{sim:rl3}
		\end{minipage}
	\end{center}
	\vspace*{-6mm}
\end{figure*}
\section{Conclusion}
\label{conclu}

Our work presents the first attempt to design online algorithm addressing the worst-case input through NNs, with the case of a classic online problem. The randomized strategy, our novel per-round update method and the powerful learning ability of NNs enable better online algorithms, as shown by our evaluation results. 
We also provide empirical evidence showing that our methods ensure 
convergence to NE. 
As future work, we are working on extending our method to algorithm design of other online problems.
\section{Acknowledgement}
This work was supported in part by grants from Hong Kong RGC under the contracts HKU 17204619, 17208920 and C5026-18G (CRF).

\bibliographystyle{ACM-Reference-Format}
\bibliography{ref}


\begin{thebibliography}{45}


\ifx \showCODEN    \undefined \def \showCODEN     #1{\unskip}     \fi
\ifx \showDOI      \undefined \def \showDOI       #1{#1}\fi
\ifx \showISBNx    \undefined \def \showISBNx     #1{\unskip}     \fi
\ifx \showISBNxiii \undefined \def \showISBNxiii  #1{\unskip}     \fi
\ifx \showISSN     \undefined \def \showISSN      #1{\unskip}     \fi
\ifx \showLCCN     \undefined \def \showLCCN      #1{\unskip}     \fi
\ifx \shownote     \undefined \def \shownote      #1{#1}          \fi
\ifx \showarticletitle \undefined \def \showarticletitle #1{#1}   \fi
\ifx \showURL      \undefined \def \showURL       {\relax}        \fi
\providecommand\bibfield[2]{#2}
\providecommand\bibinfo[2]{#2}
\providecommand\natexlab[1]{#1}
\providecommand\showeprint[2][]{arXiv:#2}

\bibitem[\protect\citeauthoryear{Arora, Hazan, and Kale}{Arora
  et~al\mbox{.}}{2012}]%
        {arora2012multiplicative}
\bibfield{author}{\bibinfo{person}{Sanjeev Arora}, \bibinfo{person}{Elad
  Hazan}, {and} \bibinfo{person}{Satyen Kale}.}
  \bibinfo{year}{2012}\natexlab{}.
\newblock \showarticletitle{The multiplicative weights update method: a
  meta-algorithm and applications}.
\newblock \bibinfo{journal}{\emph{Theory of Computing}} \bibinfo{volume}{8},
  \bibinfo{number}{1} (\bibinfo{year}{2012}), \bibinfo{pages}{121--164}.
\newblock


\bibitem[\protect\citeauthoryear{Arulkumaran, Deisenroth, Brundage, and
  Bharath}{Arulkumaran et~al\mbox{.}}{2017}]%
        {arulkumaran2017brief}
\bibfield{author}{\bibinfo{person}{Kai Arulkumaran},
  \bibinfo{person}{Marc~Peter Deisenroth}, \bibinfo{person}{Miles Brundage},
  {and} \bibinfo{person}{Anil~Anthony Bharath}.}
  \bibinfo{year}{2017}\natexlab{}.
\newblock \showarticletitle{A brief survey of deep reinforcement learning}.
\newblock \bibinfo{journal}{\emph{arXiv preprint arXiv:1708.05866}}
  (\bibinfo{year}{2017}).
\newblock


\bibitem[\protect\citeauthoryear{Banerjee, Gurvich, and Vera}{Banerjee
  et~al\mbox{.}}{2020}]%
        {banerjee2020constant}
\bibfield{author}{\bibinfo{person}{Siddhartha Banerjee}, \bibinfo{person}{Itai
  Gurvich}, {and} \bibinfo{person}{Alberto Vera}.}
  \bibinfo{year}{2020}\natexlab{}.
\newblock \showarticletitle{Constant Regret in Online Allocation: On the
  Sufficiency of a Single Historical Trace}.
\newblock  (\bibinfo{year}{2020}).
\newblock


\bibitem[\protect\citeauthoryear{Bello, Pham, Le, Norouzi, and Bengio}{Bello
  et~al\mbox{.}}{2016}]%
        {bello2016neural}
\bibfield{author}{\bibinfo{person}{Irwan Bello}, \bibinfo{person}{Hieu Pham},
  \bibinfo{person}{Quoc~V Le}, \bibinfo{person}{Mohammad Norouzi}, {and}
  \bibinfo{person}{Samy Bengio}.} \bibinfo{year}{2016}\natexlab{}.
\newblock \showarticletitle{Neural combinatorial optimization with
  reinforcement learning}.
\newblock \bibinfo{journal}{\emph{arXiv preprint arXiv:1611.09940}}
  (\bibinfo{year}{2016}).
\newblock


\bibitem[\protect\citeauthoryear{Ben-David, Borodin, Karp, Tardos, and
  Wigderson}{Ben-David et~al\mbox{.}}{1994}]%
        {ben1994power}
\bibfield{author}{\bibinfo{person}{Shai Ben-David}, \bibinfo{person}{Allan
  Borodin}, \bibinfo{person}{Richard Karp}, \bibinfo{person}{Gabor Tardos},
  {and} \bibinfo{person}{Avi Wigderson}.} \bibinfo{year}{1994}\natexlab{}.
\newblock \showarticletitle{On the power of randomization in on-line
  algorithms}.
\newblock \bibinfo{journal}{\emph{Algorithmica}} \bibinfo{volume}{11},
  \bibinfo{number}{1} (\bibinfo{year}{1994}), \bibinfo{pages}{2--14}.
\newblock


\bibitem[\protect\citeauthoryear{Brock, Donahue, and Simonyan}{Brock
  et~al\mbox{.}}{2018}]%
        {brock2018large}
\bibfield{author}{\bibinfo{person}{Andrew Brock}, \bibinfo{person}{Jeff
  Donahue}, {and} \bibinfo{person}{Karen Simonyan}.}
  \bibinfo{year}{2018}\natexlab{}.
\newblock \showarticletitle{Large scale gan training for high fidelity natural
  image synthesis}.
\newblock \bibinfo{journal}{\emph{arXiv preprint arXiv:1809.11096}}
  (\bibinfo{year}{2018}).
\newblock


\bibitem[\protect\citeauthoryear{Buchbinder, Jain, and Naor}{Buchbinder
  et~al\mbox{.}}{2007}]%
        {buchbinder2007online}
\bibfield{author}{\bibinfo{person}{Niv Buchbinder}, \bibinfo{person}{Kamal
  Jain}, {and} \bibinfo{person}{Joseph~Seffi Naor}.}
  \bibinfo{year}{2007}\natexlab{}.
\newblock \showarticletitle{Online primal-dual algorithms for maximizing
  ad-auctions revenue}. In \bibinfo{booktitle}{\emph{European Symposium on
  Algorithms}}. Springer, \bibinfo{pages}{253--264}.
\newblock


\bibitem[\protect\citeauthoryear{Buchbinder and Naor}{Buchbinder and
  Naor}{2009}]%
        {buchbinder2009online}
\bibfield{author}{\bibinfo{person}{Niv Buchbinder} {and}
  \bibinfo{person}{Joseph Naor}.} \bibinfo{year}{2009}\natexlab{}.
\newblock \showarticletitle{Online primal-dual algorithms for covering and
  packing}.
\newblock \bibinfo{journal}{\emph{Mathematics of Operations Research}}
  \bibinfo{volume}{34}, \bibinfo{number}{2} (\bibinfo{year}{2009}),
  \bibinfo{pages}{270--286}.
\newblock


\bibitem[\protect\citeauthoryear{Buchbinder, Naor, et~al\mbox{.}}{Buchbinder
  et~al\mbox{.}}{2009}]%
        {buchbinder2009design}
\bibfield{author}{\bibinfo{person}{Niv Buchbinder},
  \bibinfo{person}{Joseph~Seffi Naor}, {et~al\mbox{.}}}
  \bibinfo{year}{2009}\natexlab{}.
\newblock \showarticletitle{The design of competitive online algorithms via a
  primal--dual approach}.
\newblock \bibinfo{journal}{\emph{Foundations and Trends{\textregistered} in
  Theoretical Computer Science}} \bibinfo{volume}{3}, \bibinfo{number}{2--3}
  (\bibinfo{year}{2009}), \bibinfo{pages}{93--263}.
\newblock


\bibitem[\protect\citeauthoryear{Celli, Ciccone, Bongo, and Gatti}{Celli
  et~al\mbox{.}}{2019}]%
        {celli2019coordination}
\bibfield{author}{\bibinfo{person}{Andrea Celli}, \bibinfo{person}{Marco
  Ciccone}, \bibinfo{person}{Raffaele Bongo}, {and} \bibinfo{person}{Nicola
  Gatti}.} \bibinfo{year}{2019}\natexlab{}.
\newblock \showarticletitle{Coordination in Adversarial Sequential Team Games
  via Multi-Agent Deep Reinforcement Learning}.
\newblock \bibinfo{journal}{\emph{arXiv preprint arXiv:1912.07712}}
  (\bibinfo{year}{2019}).
\newblock


\bibitem[\protect\citeauthoryear{Du, Wu, and Huang}{Du et~al\mbox{.}}{2019}]%
        {du2019learning}
\bibfield{author}{\bibinfo{person}{Bingqian Du}, \bibinfo{person}{Chuan Wu},
  {and} \bibinfo{person}{Zhiyi Huang}.} \bibinfo{year}{2019}\natexlab{}.
\newblock \showarticletitle{Learning Resource Allocation and Pricing for Cloud
  Profit Maximization}. In \bibinfo{booktitle}{\emph{The Thirty-Third AAAI
  Conference on Artificial Intelligence (AAAI-19)}}.
\newblock


\bibitem[\protect\citeauthoryear{Gollapudi and Panigrahi}{Gollapudi and
  Panigrahi}{2019}]%
        {gollapudi2019online}
\bibfield{author}{\bibinfo{person}{Sreenivas Gollapudi} {and}
  \bibinfo{person}{Debmalya Panigrahi}.} \bibinfo{year}{2019}\natexlab{}.
\newblock \showarticletitle{Online algorithms for rent-or-buy with expert
  advice}. In \bibinfo{booktitle}{\emph{International Conference on Machine
  Learning}}. PMLR, \bibinfo{pages}{2319--2327}.
\newblock


\bibitem[\protect\citeauthoryear{Goodfellow}{Goodfellow}{2016}]%
        {goodfellow2016nips}
\bibfield{author}{\bibinfo{person}{Ian Goodfellow}.}
  \bibinfo{year}{2016}\natexlab{}.
\newblock \showarticletitle{NIPS 2016 tutorial: Generative adversarial
  networks}.
\newblock \bibinfo{journal}{\emph{arXiv preprint arXiv:1701.00160}}
  (\bibinfo{year}{2016}).
\newblock


\bibitem[\protect\citeauthoryear{Goodfellow, Pouget-Abadie, Mirza, Xu,
  Warde-Farley, Ozair, Courville, and Bengio}{Goodfellow et~al\mbox{.}}{2014}]%
        {goodfellow2014generative}
\bibfield{author}{\bibinfo{person}{Ian Goodfellow}, \bibinfo{person}{Jean
  Pouget-Abadie}, \bibinfo{person}{Mehdi Mirza}, \bibinfo{person}{Bing Xu},
  \bibinfo{person}{David Warde-Farley}, \bibinfo{person}{Sherjil Ozair},
  \bibinfo{person}{Aaron Courville}, {and} \bibinfo{person}{Yoshua Bengio}.}
  \bibinfo{year}{2014}\natexlab{}.
\newblock \showarticletitle{Generative adversarial nets}. In
  \bibinfo{booktitle}{\emph{Advances in neural information processing
  systems}}. \bibinfo{pages}{2672--2680}.
\newblock


\bibitem[\protect\citeauthoryear{Jansen and Solis-Oba}{Jansen and
  Solis-Oba}{2011}]%
        {jansen2011approximation}
\bibfield{author}{\bibinfo{person}{Klaus Jansen} {and} \bibinfo{person}{Roberto
  Solis-Oba}.} \bibinfo{year}{2011}\natexlab{}.
\newblock \bibinfo{booktitle}{\emph{Approximation and Online Algorithms: 8th
  International Workshop, WAOA 2010, Liverpool, UK, September 9-10, 2010,
  Revised Papers}}. Vol.~\bibinfo{volume}{6534}.
\newblock \bibinfo{publisher}{Springer}.
\newblock


\bibitem[\protect\citeauthoryear{Karalias and Loukas}{Karalias and
  Loukas}{2020}]%
        {karalias2020erdos}
\bibfield{author}{\bibinfo{person}{Nikolaos Karalias} {and}
  \bibinfo{person}{Andreas Loukas}.} \bibinfo{year}{2020}\natexlab{}.
\newblock \showarticletitle{Erdos goes neural: an unsupervised learning
  framework for combinatorial optimization on graphs}.
\newblock \bibinfo{journal}{\emph{arXiv preprint arXiv:2006.10643}}
  (\bibinfo{year}{2020}).
\newblock


\bibitem[\protect\citeauthoryear{Kellerer, Pferschy, and Pisinger}{Kellerer
  et~al\mbox{.}}{2004}]%
        {kellerer2004multidimensional}
\bibfield{author}{\bibinfo{person}{Hans Kellerer}, \bibinfo{person}{Ulrich
  Pferschy}, {and} \bibinfo{person}{David Pisinger}.}
  \bibinfo{year}{2004}\natexlab{}.
\newblock \showarticletitle{Multidimensional knapsack problems}.
\newblock In \bibinfo{booktitle}{\emph{Knapsack problems}}.
  \bibinfo{publisher}{Springer}, \bibinfo{pages}{235--283}.
\newblock


\bibitem[\protect\citeauthoryear{Khalil, Dai, Zhang, Dilkina, and Song}{Khalil
  et~al\mbox{.}}{2017}]%
        {khalil2017learning}
\bibfield{author}{\bibinfo{person}{Elias Khalil}, \bibinfo{person}{Hanjun Dai},
  \bibinfo{person}{Yuyu Zhang}, \bibinfo{person}{Bistra Dilkina}, {and}
  \bibinfo{person}{Le Song}.} \bibinfo{year}{2017}\natexlab{}.
\newblock \showarticletitle{Learning combinatorial optimization algorithms over
  graphs}. In \bibinfo{booktitle}{\emph{Advances in neural information
  processing systems}}. \bibinfo{pages}{6348--6358}.
\newblock


\bibitem[\protect\citeauthoryear{Kool, Van~Hoof, and Welling}{Kool
  et~al\mbox{.}}{2018}]%
        {kool2018attention}
\bibfield{author}{\bibinfo{person}{Wouter Kool}, \bibinfo{person}{Herke
  Van~Hoof}, {and} \bibinfo{person}{Max Welling}.}
  \bibinfo{year}{2018}\natexlab{}.
\newblock \showarticletitle{Attention, learn to solve routing problems!}
\newblock \bibinfo{journal}{\emph{arXiv preprint arXiv:1803.08475}}
  (\bibinfo{year}{2018}).
\newblock


\bibitem[\protect\citeauthoryear{Kumar, Purohit, and Svitkina}{Kumar
  et~al\mbox{.}}{2018}]%
        {kumar2018improving}
\bibfield{author}{\bibinfo{person}{Ravi Kumar}, \bibinfo{person}{Manish
  Purohit}, {and} \bibinfo{person}{Zoya Svitkina}.}
  \bibinfo{year}{2018}\natexlab{}.
\newblock \showarticletitle{Improving online algorithms via ML predictions}. In
  \bibinfo{booktitle}{\emph{Proceedings of the 32nd International Conference on
  Neural Information Processing Systems}}. \bibinfo{pages}{9684--9693}.
\newblock


\bibitem[\protect\citeauthoryear{Letcher, Balduzzi, Racaniere, Martens,
  Foerster, Tuyls, and Graepel}{Letcher et~al\mbox{.}}{2019}]%
        {letcher2019differentiable}
\bibfield{author}{\bibinfo{person}{Alistair Letcher}, \bibinfo{person}{David
  Balduzzi}, \bibinfo{person}{S{\'e}bastien Racaniere}, \bibinfo{person}{James
  Martens}, \bibinfo{person}{Jakob Foerster}, \bibinfo{person}{Karl Tuyls},
  {and} \bibinfo{person}{Thore Graepel}.} \bibinfo{year}{2019}\natexlab{}.
\newblock \showarticletitle{Differentiable game mechanics}.
\newblock \bibinfo{journal}{\emph{The Journal of Machine Learning Research}}
  \bibinfo{volume}{20}, \bibinfo{number}{1} (\bibinfo{year}{2019}),
  \bibinfo{pages}{3032--3071}.
\newblock


\bibitem[\protect\citeauthoryear{Li, Chen, and Koltun}{Li
  et~al\mbox{.}}{2018}]%
        {li2018combinatorial}
\bibfield{author}{\bibinfo{person}{Zhuwen Li}, \bibinfo{person}{Qifeng Chen},
  {and} \bibinfo{person}{Vladlen Koltun}.} \bibinfo{year}{2018}\natexlab{}.
\newblock \showarticletitle{Combinatorial optimization with graph convolutional
  networks and guided tree search}. In \bibinfo{booktitle}{\emph{Advances in
  Neural Information Processing Systems}}. \bibinfo{pages}{539--548}.
\newblock


\bibitem[\protect\citeauthoryear{Liang and Stokes}{Liang and Stokes}{2019}]%
        {liang2019interaction}
\bibfield{author}{\bibinfo{person}{Tengyuan Liang} {and} \bibinfo{person}{James
  Stokes}.} \bibinfo{year}{2019}\natexlab{}.
\newblock \showarticletitle{Interaction matters: A note on non-asymptotic local
  convergence of generative adversarial networks}. In
  \bibinfo{booktitle}{\emph{The 22nd International Conference on Artificial
  Intelligence and Statistics}}. PMLR, \bibinfo{pages}{907--915}.
\newblock


\bibitem[\protect\citeauthoryear{Luo, Yang, and Liu}{Luo et~al\mbox{.}}{2020}]%
        {luo2020policy}
\bibfield{author}{\bibinfo{person}{Biao Luo}, \bibinfo{person}{Yin Yang}, {and}
  \bibinfo{person}{Derong Liu}.} \bibinfo{year}{2020}\natexlab{}.
\newblock \showarticletitle{Policy Iteration Q-Learning for Data-Based
  Two-Player Zero-Sum Game of Linear Discrete-Time Systems}.
\newblock \bibinfo{journal}{\emph{IEEE Transactions on Cybernetics}}
  (\bibinfo{year}{2020}).
\newblock


\bibitem[\protect\citeauthoryear{Lykouris and Vassilvtiskii}{Lykouris and
  Vassilvtiskii}{2018}]%
        {lykouris2018competitive}
\bibfield{author}{\bibinfo{person}{Thodoris Lykouris} {and}
  \bibinfo{person}{Sergei Vassilvtiskii}.} \bibinfo{year}{2018}\natexlab{}.
\newblock \showarticletitle{Competitive caching with machine learned advice}.
  In \bibinfo{booktitle}{\emph{International Conference on Machine Learning}}.
  PMLR, \bibinfo{pages}{3296--3305}.
\newblock


\bibitem[\protect\citeauthoryear{Mao, Alizadeh, Menache, and Kandula}{Mao
  et~al\mbox{.}}{2016}]%
        {mao2016resource}
\bibfield{author}{\bibinfo{person}{Hongzi Mao}, \bibinfo{person}{Mohammad
  Alizadeh}, \bibinfo{person}{Ishai Menache}, {and} \bibinfo{person}{Srikanth
  Kandula}.} \bibinfo{year}{2016}\natexlab{}.
\newblock \showarticletitle{Resource management with deep reinforcement
  learning}. In \bibinfo{booktitle}{\emph{Proceedings of the 15th ACM Workshop
  on Hot Topics in Networks}}. \bibinfo{pages}{50--56}.
\newblock


\bibitem[\protect\citeauthoryear{Marchetti-Spaccamela and
  Vercellis}{Marchetti-Spaccamela and Vercellis}{1995}]%
        {10.1007/BF01585758}
\bibfield{author}{\bibinfo{person}{A. Marchetti-Spaccamela} {and}
  \bibinfo{person}{C. Vercellis}.} \bibinfo{year}{1995}\natexlab{}.
\newblock \showarticletitle{Stochastic On-Line Knapsack Problems}.
\newblock \bibinfo{journal}{\emph{Math. Program.}} \bibinfo{volume}{68},
  \bibinfo{number}{1} (\bibinfo{date}{Jan.} \bibinfo{year}{1995}),
  \bibinfo{pages}{73–104}.
\newblock
\showISSN{0025-5610}
\urldef\tempurl%
\url{https://doi.org/10.1007/BF01585758}
\showDOI{\tempurl}


\bibitem[\protect\citeauthoryear{Medina and Vassilvitskii}{Medina and
  Vassilvitskii}{2017}]%
        {medina2017revenue}
\bibfield{author}{\bibinfo{person}{Andr{\'e}s~Munoz Medina} {and}
  \bibinfo{person}{Sergei Vassilvitskii}.} \bibinfo{year}{2017}\natexlab{}.
\newblock \showarticletitle{Revenue optimization with approximate bid
  predictions}.
\newblock \bibinfo{journal}{\emph{arXiv preprint arXiv:1706.04732}}
  (\bibinfo{year}{2017}).
\newblock


\bibitem[\protect\citeauthoryear{Mirza and Osindero}{Mirza and
  Osindero}{2014}]%
        {mirza2014conditional}
\bibfield{author}{\bibinfo{person}{Mehdi Mirza} {and} \bibinfo{person}{Simon
  Osindero}.} \bibinfo{year}{2014}\natexlab{}.
\newblock \showarticletitle{Conditional generative adversarial nets}.
\newblock \bibinfo{journal}{\emph{arXiv preprint arXiv:1411.1784}}
  (\bibinfo{year}{2014}).
\newblock


\bibitem[\protect\citeauthoryear{Nagarajan and Kolter}{Nagarajan and
  Kolter}{2017}]%
        {nagarajan2017gradient}
\bibfield{author}{\bibinfo{person}{Vaishnavh Nagarajan} {and}
  \bibinfo{person}{J~Zico Kolter}.} \bibinfo{year}{2017}\natexlab{}.
\newblock \showarticletitle{Gradient descent GAN optimization is locally
  stable}.
\newblock \bibinfo{journal}{\emph{arXiv preprint arXiv:1706.04156}}
  (\bibinfo{year}{2017}).
\newblock


\bibitem[\protect\citeauthoryear{Nie, Trullo, Lian, Petitjean, Ruan, Wang, and
  Shen}{Nie et~al\mbox{.}}{2017}]%
        {nie2017medical}
\bibfield{author}{\bibinfo{person}{Dong Nie}, \bibinfo{person}{Roger Trullo},
  \bibinfo{person}{Jun Lian}, \bibinfo{person}{Caroline Petitjean},
  \bibinfo{person}{Su Ruan}, \bibinfo{person}{Qian Wang}, {and}
  \bibinfo{person}{Dinggang Shen}.} \bibinfo{year}{2017}\natexlab{}.
\newblock \showarticletitle{Medical image synthesis with context-aware
  generative adversarial networks}. In \bibinfo{booktitle}{\emph{International
  Conference on Medical Image Computing and Computer-Assisted Intervention}}.
  Springer, \bibinfo{pages}{417--425}.
\newblock


\bibitem[\protect\citeauthoryear{Prodan and Nae}{Prodan and Nae}{2009}]%
        {prodan2009prediction}
\bibfield{author}{\bibinfo{person}{Radu Prodan} {and} \bibinfo{person}{Vlad
  Nae}.} \bibinfo{year}{2009}\natexlab{}.
\newblock \showarticletitle{Prediction-based real-time resource provisioning
  for massively multiplayer online games}.
\newblock \bibinfo{journal}{\emph{Future Generation Computer Systems}}
  \bibinfo{volume}{25}, \bibinfo{number}{7} (\bibinfo{year}{2009}),
  \bibinfo{pages}{785--793}.
\newblock


\bibitem[\protect\citeauthoryear{Silver, Schrittwieser, Simonyan, Antonoglou,
  Huang, Guez, Hubert, Baker, Lai, Bolton, et~al\mbox{.}}{Silver
  et~al\mbox{.}}{2017}]%
        {silver2017mastering}
\bibfield{author}{\bibinfo{person}{David Silver}, \bibinfo{person}{Julian
  Schrittwieser}, \bibinfo{person}{Karen Simonyan}, \bibinfo{person}{Ioannis
  Antonoglou}, \bibinfo{person}{Aja Huang}, \bibinfo{person}{Arthur Guez},
  \bibinfo{person}{Thomas Hubert}, \bibinfo{person}{Lucas Baker},
  \bibinfo{person}{Matthew Lai}, \bibinfo{person}{Adrian Bolton},
  {et~al\mbox{.}}} \bibinfo{year}{2017}\natexlab{}.
\newblock \showarticletitle{Mastering the game of go without human knowledge}.
\newblock \bibinfo{journal}{\emph{nature}} \bibinfo{volume}{550},
  \bibinfo{number}{7676} (\bibinfo{year}{2017}), \bibinfo{pages}{354--359}.
\newblock


\bibitem[\protect\citeauthoryear{Simon, Blume, et~al\mbox{.}}{Simon
  et~al\mbox{.}}{1994}]%
        {simon1994mathematics}
\bibfield{author}{\bibinfo{person}{Carl~P Simon}, \bibinfo{person}{Lawrence
  Blume}, {et~al\mbox{.}}} \bibinfo{year}{1994}\natexlab{}.
\newblock \bibinfo{booktitle}{\emph{Mathematics for economists}}.
  Vol.~\bibinfo{volume}{7}.
\newblock \bibinfo{publisher}{Norton New York}.
\newblock


\bibitem[\protect\citeauthoryear{Singh, Kearns, and Mansour}{Singh
  et~al\mbox{.}}{2013}]%
        {singh2013nash}
\bibfield{author}{\bibinfo{person}{Satinder Singh}, \bibinfo{person}{Michael
  Kearns}, {and} \bibinfo{person}{Yishay Mansour}.}
  \bibinfo{year}{2013}\natexlab{}.
\newblock \showarticletitle{Nash convergence of gradient dynamics in iterated
  general-sum games}.
\newblock \bibinfo{journal}{\emph{arXiv preprint arXiv:1301.3892}}
  (\bibinfo{year}{2013}).
\newblock


\bibitem[\protect\citeauthoryear{Sleator and Tarjan}{Sleator and
  Tarjan}{1985}]%
        {sleator1985amortized}
\bibfield{author}{\bibinfo{person}{Daniel~D Sleator} {and}
  \bibinfo{person}{Robert~E Tarjan}.} \bibinfo{year}{1985}\natexlab{}.
\newblock \showarticletitle{Amortized efficiency of list update and paging
  rules}.
\newblock \bibinfo{journal}{\emph{Commun. ACM}} \bibinfo{volume}{28},
  \bibinfo{number}{2} (\bibinfo{year}{1985}), \bibinfo{pages}{202--208}.
\newblock


\bibitem[\protect\citeauthoryear{Steinberger, Lerer, and Brown}{Steinberger
  et~al\mbox{.}}{2020}]%
        {steinberger2020dream}
\bibfield{author}{\bibinfo{person}{Eric Steinberger}, \bibinfo{person}{Adam
  Lerer}, {and} \bibinfo{person}{Noam Brown}.} \bibinfo{year}{2020}\natexlab{}.
\newblock \showarticletitle{DREAM: Deep regret minimization with advantage
  baselines and model-free learning}.
\newblock \bibinfo{journal}{\emph{arXiv preprint arXiv:2006.10410}}
  (\bibinfo{year}{2020}).
\newblock


\bibitem[\protect\citeauthoryear{Tesauro et~al\mbox{.}}{Tesauro
  et~al\mbox{.}}{2005}]%
        {tesauro2005online}
\bibfield{author}{\bibinfo{person}{Gerald Tesauro} {et~al\mbox{.}}}
  \bibinfo{year}{2005}\natexlab{}.
\newblock \showarticletitle{Online resource allocation using decompositional
  reinforcement learning}. In \bibinfo{booktitle}{\emph{AAAI}},
  Vol.~\bibinfo{volume}{5}. \bibinfo{pages}{886--891}.
\newblock


\bibitem[\protect\citeauthoryear{Vera, Banerjee, and Gurvich}{Vera
  et~al\mbox{.}}{2021}]%
        {vera2021online}
\bibfield{author}{\bibinfo{person}{Alberto Vera}, \bibinfo{person}{Siddhartha
  Banerjee}, {and} \bibinfo{person}{Itai Gurvich}.}
  \bibinfo{year}{2021}\natexlab{}.
\newblock \showarticletitle{Online allocation and pricing: Constant regret via
  bellman inequalities}.
\newblock \bibinfo{journal}{\emph{Operations Research}} (\bibinfo{year}{2021}).
\newblock


\bibitem[\protect\citeauthoryear{Wang, Shi, Yu, Wu, Singh, Joppa, and
  Fang}{Wang et~al\mbox{.}}{2019}]%
        {wang2019deep}
\bibfield{author}{\bibinfo{person}{Yufei Wang}, \bibinfo{person}{Zheyuan~Ryan
  Shi}, \bibinfo{person}{Lantao Yu}, \bibinfo{person}{Yi Wu},
  \bibinfo{person}{Rohit Singh}, \bibinfo{person}{Lucas Joppa}, {and}
  \bibinfo{person}{Fei Fang}.} \bibinfo{year}{2019}\natexlab{}.
\newblock \showarticletitle{Deep reinforcement learning for green security
  games with real-time information}. In \bibinfo{booktitle}{\emph{Proceedings
  of the AAAI Conference on Artificial Intelligence}},
  Vol.~\bibinfo{volume}{33}. \bibinfo{pages}{1401--1408}.
\newblock


\bibitem[\protect\citeauthoryear{Wang, Gwon, Oates, and Iezzi}{Wang
  et~al\mbox{.}}{2017}]%
        {wang2017automated}
\bibfield{author}{\bibinfo{person}{Zhiguang Wang}, \bibinfo{person}{Chul Gwon},
  \bibinfo{person}{Tim Oates}, {and} \bibinfo{person}{Adam Iezzi}.}
  \bibinfo{year}{2017}\natexlab{}.
\newblock \showarticletitle{Automated cloud provisioning on aws using deep
  reinforcement learning}.
\newblock \bibinfo{journal}{\emph{arXiv preprint arXiv:1709.04305}}
  (\bibinfo{year}{2017}).
\newblock


\bibitem[\protect\citeauthoryear{Yu, Zhang, Wang, and Yu}{Yu
  et~al\mbox{.}}{2017}]%
        {yu2017seqgan}
\bibfield{author}{\bibinfo{person}{Lantao Yu}, \bibinfo{person}{Weinan Zhang},
  \bibinfo{person}{Jun Wang}, {and} \bibinfo{person}{Yong Yu}.}
  \bibinfo{year}{2017}\natexlab{}.
\newblock \showarticletitle{Seqgan: Sequence generative adversarial nets with
  policy gradient}. In \bibinfo{booktitle}{\emph{Thirty-First AAAI Conference
  on Artificial Intelligence}}.
\newblock


\bibitem[\protect\citeauthoryear{Zhang, Gan, and Carin}{Zhang
  et~al\mbox{.}}{2016}]%
        {zhang2016generating}
\bibfield{author}{\bibinfo{person}{Yizhe Zhang}, \bibinfo{person}{Zhe Gan},
  {and} \bibinfo{person}{Lawrence Carin}.} \bibinfo{year}{2016}\natexlab{}.
\newblock \showarticletitle{Generating text via adversarial training}. In
  \bibinfo{booktitle}{\emph{NIPS workshop on Adversarial Training}},
  Vol.~\bibinfo{volume}{21}.
\newblock


\bibitem[\protect\citeauthoryear{Zhang, Li, and Wu}{Zhang
  et~al\mbox{.}}{2017}]%
        {zhang2017optimal}
\bibfield{author}{\bibinfo{person}{Zijun Zhang}, \bibinfo{person}{Zongpeng Li},
  {and} \bibinfo{person}{Chuan Wu}.} \bibinfo{year}{2017}\natexlab{}.
\newblock \showarticletitle{Optimal posted prices for online cloud resource
  allocation}.
\newblock \bibinfo{journal}{\emph{Proceedings of the ACM on Measurement and
  Analysis of Computing Systems}} \bibinfo{volume}{1}, \bibinfo{number}{1}
  (\bibinfo{year}{2017}), \bibinfo{pages}{23}.
\newblock


\bibitem[\protect\citeauthoryear{Zhou, Chakrabarty, and Lukose}{Zhou
  et~al\mbox{.}}{2008}]%
        {zhou2008budget}
\bibfield{author}{\bibinfo{person}{Yunhong Zhou}, \bibinfo{person}{Deeparnab
  Chakrabarty}, {and} \bibinfo{person}{Rajan Lukose}.}
  \bibinfo{year}{2008}\natexlab{}.
\newblock \showarticletitle{Budget constrained bidding in keyword auctions and
  online knapsack problems}. In \bibinfo{booktitle}{\emph{International
  Workshop on Internet and Network Economics}}. Springer,
  \bibinfo{pages}{566--576}.
\newblock


\end{thebibliography}


\end{document}